\documentclass{article}

\usepackage{microtype}
\usepackage{graphicx}
\usepackage{subfigure}
\usepackage{booktabs} 

\usepackage{hyperref}

\usepackage[linesnumbered,algoruled,boxed,lined,algo2e]{algorithm2e}

\usepackage[accepted]{icml2019}

\usepackage{enumerate}
\usepackage{natbib}
\usepackage{multirow}
\usepackage{makecell}
\usepackage{amsmath, amsthm, amssymb,graphicx}

\DeclareMathOperator*{\argmin}{arg\,min}
\DeclareMathOperator*{\argmax}{arg\,max}

\usepackage{stackengine}

\makeatletter
\newcommand{\distas}[1]{\mathbin{\overset{#1}{\kern\z@\sim}}}%

\newtheorem{theorem}{Theorem}

\newtheorem{cor}[theorem]{Corollary}
\newtheorem{proposition}[theorem]{Proposition}

\newcommand{\bas}[1]{\begin{align*}#1\end{align*}}
\newcommand{\ba}[1]{\begin{align}#1\end{align}}
\newcommand{\baa}[1]{\begin{equation}\begin{aligned}#1\end{aligned}\end{equation}}

\newcommand{\bE}{\mathbb{E}}

\newcommand{\cE}{\mathcal{E}}

\newcommand{\dir}{\text{Dirichlet}}

\newcommand{\beqs}{\vspace{0mm}\begin{eqnarray}}
\newcommand{\eeqs}{\vspace{0mm}\end{eqnarray}}
\newcommand{\barr}{\begin{array}}
\newcommand{\earr}{\end{array}}

\newcommand{\bv}[0]{{\boldsymbol{b}}}

\newcommand{\jv}[0]{{\boldsymbol{j}} }

\newcommand{\sv}[0]{{\boldsymbol{s}}}

\newcommand{\wv}{\boldsymbol{w}}
\newcommand{\xv}{\boldsymbol{x}}

\newcommand{\zv}{\boldsymbol{z}}

\newcommand{\cdotv}{\boldsymbol{\cdot}}

\newcommand{\Pimat}[0]{{\boldsymbol{\Pi}} }

\newcommand{\Phimat}{\boldsymbol{\Phi}}

\newcommand{\epsilonv}{\boldsymbol{\epsilon}}

\newcommand{\thetav}{\boldsymbol{\theta}}

\newcommand{\piv}{\boldsymbol{\pi}}

\newcommand{\phiv}{\boldsymbol{\phi}}

\newcommand{\varpiv}[0]{{\boldsymbol{\varpi}} }

\newcommand{\E}{\mathbb{E}}

\newcommand{\given}{\,|\,}

\makeatletter
\providecommand{\leftsquigarrow}{%
 \mathrel{\mathpalette\reflect@squig\relax}%
}
\newcommand{\reflect@squig}[2]{%
 \reflectbox{$\m@th#1\rightsquigarrow$}%
}
\makeatother

\icmltitlerunning{ARSM: Augment-REINFORCE-Swap-Merge Gradient for Categorical Variables} 

\begin{document}

\twocolumn[
\icmltitle{
ARSM: Augment-REINFORCE-Swap-Merge Estimator for \\
Gradient Backpropagation Through Categorical Variables} 



\icmlsetsymbol{equal}{*}

\begin{icmlauthorlist}
\icmlauthor{Mingzhang Yin}{equal,to}
\icmlauthor{Yuguang Yue}{equal,to}
\icmlauthor{Mingyuan Zhou}{goo}
\end{icmlauthorlist}

\icmlaffiliation{to}{Department of Statistics and Data Sciences,}
\icmlaffiliation{goo}{Department of IROM, McCombs School of Business, The University of Texas at Austin, Austin, TX 78712, USA}
\icmlcorrespondingauthor{Mingyuan Zhou}{mingyuan.zhou@mccombs.utexas.edu}

\icmlkeywords{Discrete data analysis, latent variable models, variational auto-encoder, discrete-action policy gradient}

\vskip 0.3in
]



\printAffiliationsAndNotice{\icmlEqualContribution} 

\begin{abstract}To address the challenge of backpropagating the gradient through categorical variables, we propose the augment-REINFORCE-swap-merge (ARSM) gradient estimator that is unbiased and has low variance. ARSM first uses variable augmentation, REINFORCE, and Rao-Blackwellization to re-express the gradient as an expectation under the Dirichlet distribution, then uses variable swapping to construct differently expressed but equivalent expectations, and finally shares common random numbers between these expectations to achieve significant variance reduction. Experimental results show ARSM closely resembles the performance of the true gradient for optimization in univariate settings; outperforms existing estimators by a large margin when applied to categorical variational auto-encoders; and provides a ``try-and-see self-critic'' variance reduction method for discrete-action policy gradient, which removes the need of estimating baselines by generating a random number of pseudo actions and estimating their action-value functions. 
\end{abstract}

\section{Introduction}
\label{submission}

The need to maximize an objective function, expressed as the expectation over categorical variables, arises in a wide variety of settings, such as discrete latent variable models \cite{BNBP_EPPF,jang2016categorical,maddison2016concrete} and 
policy optimization for reinforcement learning (RL) with discrete actions \cite{sutton2018reinforcement,weaver2001optimal,schulman2015trust,mnih2016asynchronous,grathwohl2017backpropagation}. 
More specifically, let us denote $z_k\in\{1,2,\ldots,{C}\}$ as a univariate ${C}$-way categorical variable, and $\zv=(z_1,\ldots,z_K)\in\{1,2,\ldots,{C}\}^K$ as a $K$-dimensional ${C}$-way multivariate categorical vector. In discrete latent variable models, $K$ will be the dimension of the discrete latent space, each dimension of which can be further represented as a ${C}$-dimensional one-hot vector. In RL, ${C}$ represents the size of the discrete action space and $\zv$ is a sequence of discrete actions from that space.
In even more challenging settings, one may have a sequence of $K$-dimensional $C$-way multivariate categorical vectors, which appear both in categorical latent variable models with multiple stochastic layers, and in RL with a high dimensional discrete action space or multiple agents, which may consist of as many as $C^K$ unique combinations at each time step. 
 
 With $f(\zv)$ and $q_{\phiv}(\zv)$ denoted as the reward function and distribution for categorical $\zv$, respectively, we need to optimize parameter $\phiv$ to maximize the expected reward as
\ba{
\mathcal{E}(\phiv)=\textstyle{\int} f{}(\zv) q_{\phiv}(\zv) d\zv = \E_{\zv\sim q_{\phiv}(\zv)} [f{}(\zv)]. \label{eq:E}
}
Here we consider both categorical latent variable models and policy optimization for discrete actions, which arise in a wide array of real-world applications. 
A number of unbiased estimators for backpropagating the gradient through discrete latent variables have been recently proposed \cite{tucker2017rebar,grathwohl2017backpropagation,ARM_ICLR2019,andriyash2018improved}. However, they all mainly, if not exclusively, focus on the binary case ($i.e.$, $C=2$). 
The categorical case ($i.e.$, $C\ge 2$) is more widely applicable but generally much more challenging. 
In this paper, to optimize the objective in \eqref{eq:E}, inspired by the augment-REINFORCE-merge (ARM) gradient estimator restricted for binary variables \citep{ARM_ICLR2019}, we introduce the augment-REINFORCE-swap-merge (ARSM) estimator that is unbiased and well controls its variance for categorical variables.

The proposed ARSM estimator combines variable augmentation \citep{tanner1987calculation,van2001art}, REINFORCE \citep{williams1992simple} in an augmented space, Rao-Blackwellization \citep{casella1996rao}, and a merge step that shares common random numbers between different but equivalent gradient expectations to achieve significant variance reduction. While ARSM with $C=2$ reduces to the ARM estimator \citep{ARM_ICLR2019}, whose merge step can be realized by applying antithetic sampling \citep{mcbook} in the augmented space, the merge step of ARSM with ${C}>2$ cannot be realized in this manner. Instead, ARSM requires distinct variable-swapping operations to construct differently expressed but equivalent expectations under the Dirichlet distribution before performing its merge step. 

Experimental results on both synthetic data and several representative tasks involving categorical variables are used to illustrate the distinct working mechanism of ARSM. 
 In particular, 
 our experimental results on latent variable models with one or multiple categorical stochastic hidden layers show that ARSM provides state-of-the-art training and
out-of-sample prediction performance. 
Our experiments on RL with discrete action spaces show that 
ARSM provides a ``try-and-see self-critic'' method to produce unbiased and low-variance policy gradient estimates, removing the need of constructing baselines by generating a random number of pseudo actions at a given state and estimating their action-value functions. 
These results demonstrate the effectiveness and versatility of the ARSM estimator for gradient backpropagation through categorical stochastic layers. Python code for reproducible research is available at \href{https://github.com/ARM-gradient/ARSM}{https://github.com/ARM-gradient/ARSM}.

\subsection{Related Work}

For optimizing \eqref{eq:E} for categorical $\zv$, the difficulty lies in developing a low-variance and preferably unbiased estimator for its gradient with respect to $\phiv$, expressed as $\nabla_{\phiv}\mathcal{E}(\phiv)$. An unbiased but high-variance gradient estimator that is universally applicable to \eqref{eq:E} is REINFORCE \citep{williams1992simple}. 
Using the score function $\nabla_{\phiv} \log q_{\phiv}(\zv) = \nabla_{\phiv} q_{\phiv}(\zv)/q_{\phiv}(\zv)$, REINFORCE expresses the gradient as an expectation as 
\ba{
\nabla_{\phiv}\mathcal{E}(\phiv)&= 
\E_{\zv\sim q_{\phiv}(\zv)} [f{}(\zv) \nabla_{\phiv} \log q_{\phiv}(\zv) ],\label{eq:E_score}
}
and approximates it with Monte Carlo integration \citep{mcbook}. However, the estimation variance with a limited number of Monte Carlo samples is often too high to make vanilla REINFORCE a sound choice for 
categorical $\zv$.

To address the high-estimation-variance issue for categorical~$\zv$, one often resorts to
a biased gradient estimator. For example, \citet{maddison2016concrete} and \citet{jang2016categorical} relax the categorical variables with continuous ones and then apply the reparameterization trick to estimate the gradients, reducing variance 
but introducing bias. 
Other biased estimators for backpropagating through binary variables include the straight-through estimator \cite{ST,bengio2013estimating} and the ones of \citet{gregor2014deep,raiko2014techniques,chengstraight}. 
With biased gradient estimates, 
however, a gradient ascent algorithm may not be guaranteed to work, or may converge to unintended solutions.

To keep REINFORCE unbiased while sufficiently reducing its variance, a usual strategy is to introduce appropriate control variates, also known as baselines \citep{williams1992simple}, into the expectation in \eqref{eq:E_score} before performing Monte Carlo integration \citep{paisley2012variational,ranganath2014black, mnih2014neural,
gu2015muprop,mnih2016variational,
ruiz2016generalized,kucukelbir2017automatic,naesseth2017reparameterization}. For discrete $\zv$, 
 \citet{tucker2017rebar} and \citet{grathwohl2017backpropagation} improve REINFORCE by introducing continuous relaxation based baselines, whose parameters are optimized by minimizing the sample variance of gradient estimates. 

\section{ ARSM Gradient For Categorical Variables} 
Let us denote $z\sim{\mbox{Cat}}(\sigma(\phiv))$ as a categorical variable such that $\textstyle P(z={c}\given \phiv) =\sigma(\phiv)_{c} ={e^{\phi_{c}}}\big/{\sum_{i=1}^{C} e^{\phi_{i}}},$
where $\phiv:=(\phi_1,\ldots,\phi_{C})$ and $\sigma(\phiv):=(e^{\phi_1},\ldots,e^{\phi_{C}})/\sum_{i=1}^{C} e^{\phi_i}$ is the softmax function. 
For the expectated reward 
defined as $$\textstyle\mathcal{E}(\phiv):=\E_{z\sim{\text{Cat}}(\sigma(\phiv))}[f{}(z)]=\sum_{i=1}^{C} f(i) \sigma(\phiv)_i,$$ the gradient can be expressed analytically as
\ba{
\textstyle \nabla_{\phi_{c}} \mathcal{E}(\phiv) 
= \sigma(\phiv)_{c} f(c)-\sigma(\phiv)_{c} \mathcal{E}(\phiv) 
\label{eq:CatGrad}}
or expressed with REINFORCE as 
\ba{
\nabla_{\phi_{c}} \mathcal{E}(\phiv) 
=\E_{z\sim{\text{Cat}}(\sigma(\phiv))}\left[f(z) (\mathbf{1}_{[z={c}]} - \sigma(\phiv)_{c} )
\right], \label{eq:REINFORCE}}
where $\mathbf{1}_{[\cdotv]}$ is an indicator function that is equal to one if the argument is true and zero otherwise. However, 
the analytic expression quickly becomes intractable for a multivariate setting, and the REINFORCE estimator often comes with significant estimation 
variance. 
While the ARM estimator of \citet{ARM_ICLR2019} is unbiased and provides significant variance reduction for binary variables, it is restricted to ${C}=2$ and hence has limited applicability. 

Below we introduce the augment-REINFORCE (AR), AR-swap (ARS), and ARS-merge (ARSM) estimators for a univariate ${C}$-way categorical variable, and later generalize them to multivariate, hierarchical, and sequential settings.

\subsection{AR: Augment-REINFORCE}
Let us denote $\piv:=(\pi_1,\ldots,\pi_{C})\sim{\text{Dir}}(\mathbf{1}_{C})$ as a Dirichlet distribution whose ${C}$ parameters are all ones. 
We 
 first state three statistical properties that can directly lead to the proposed AR estimator. 
 We describe in detail in Appendix~\ref{sec:derivation} how we actually arrive at the AR estimator, with these properties obtained as by-products, by performing variable augmentation, REINFORCE, and Rao-Blackwellization. Thus we are in fact reverse-engineering our original derivation of the AR estimator to help concisely present 
 our findings. 
 
\emph{\textbf{Property I.} The categorical variable $z\sim\emph{\mbox{Cat}}(\sigma(\phiv))$ can be equivalently generated as} $$\textstyle z:=\argmin_{i\in\{1,\ldots,{C}\}} \pi_i e^{-\phi_i},~\piv\sim{\text{Dir}}(\mathbf{1}_{C}).$$
\emph{\textbf{Property II.}} 
$
 \mathcal{E}(\phiv) = \E_{\piv\sim{\text{Dir}}(\mathbf{1}_{C})}[f( \argmin_i \pi_i e^{-\phi_i})].
$\\
\textbf{\emph{Property\,III}.} 
$\E_{ \piv\sim{\text{Dir}}(\mathbf{1}_{C})}[f(\argmin\nolimits_{i} \pi_i e^{-\phi_i} ){C}\pi_{c}] \textstyle= \mathcal{E}(\phiv)+ \sigma(\phiv)_{c} \mathcal{E}(\phiv)
- \sigma(\phiv)_{c} f(c). 
$

These three properties, Property III in particular, are previously unknown to the best of our knowledge. They are directly linked to the AR estimator shown below.
\begin{theorem}[AR estimator]\label{th:AR} 
The gradient of $\mathcal{E}(\phiv)=\E_{z\sim\emph{\text{Cat}}(\sigma(\phiv))}[f{}(z)] $, as shown in \eqref{eq:CatGrad}, 
can be re-expressed as an expectation under a Dirichlet distribution as
 \begin{equation}
\begin{aligned}
\nabla_{\phi_{c}} \mathcal{E}(\phiv) & = \E_{ \piv\sim\emph{\text{Dir}}(\mathbf{1}_{C})}[g_{\emph{\text{AR}}}(\piv)_c],\\
g_{\emph{\text{AR}}}(\piv)_c:&=f(z )(1-{C} \pi_{c}),\\
z:&=\argmin\nolimits_{i\in\{1,\ldots,{C}\}} \pi_i e^{-\phi_i}.
\end{aligned}\label{eq:Grad_AR}
 \end{equation}
\end{theorem}
 
Distinct from REINFORCE in \eqref{eq:REINFORCE}, the AR estimator in \eqref{eq:Grad_AR} now expresses the gradient as an expectation under a Dirichlet distributed random noise. From this point of view, it is somewhat related to the reparameterization trick \citep{kingma2013auto,rezende2014stochastic}, which is widely used to express the gradient of an expectation under reparameterizable random variables as an expectation under random noises.
Thus one may consider AR as a special type of reparameterization gradient, which, however, requires neither $z$ to be reparameterizable nor $f(\cdotv)$ to be differentiable.

 \subsection{ARS: Augment-REINFORCE-Swap}

Let us swap the $m$th and $j$th elements of $\piv$ to define vector 
$$
\piv^{_{m \leftrightharpoons j}}:=(\pi^{_{m \leftrightharpoons j}}_1,\ldots,\pi^{_{m \leftrightharpoons j}}_{C}), $$
where $\pi^{_{m \leftrightharpoons j}}_m = \pi_j$, $\pi^{_{m \leftrightharpoons j}}_j = \pi_m$, and $\forall ~{c}\notin\{m,j\}$, $\pi^{_{m \leftrightharpoons j}}_{c}= {\pi_c}$. Another property to be repeatedly used is: 

\emph{\textbf{Property IV.}
If} $\piv\sim{\text{Dir}}(\mathbf{1}_{C})$, \emph{then} $\piv^{_{m \leftrightharpoons j}}\sim{\text{Dir}}(\mathbf{1}_{C})$.

This leads to
 a key observation for the AR estimator in~\eqref{eq:Grad_AR}: swapping any two variables of the probability vector $\piv$ inside the expectation does not change the expected value. 
Using the idea of sharing common random numbers between different expectations to potentially significantly reduce Monte Carlo integration variance \citep{mcbook}, we propose to swap $\pi_{c}$ and $\pi_j$ in \eqref{eq:Grad_AR}, where $ j\in\{1,\ldots,{C}\}$ is a reference category chosen independently of $\piv$ and $\phiv$. This variable-swapping operation changes 
the AR estimator to 
 \begin{equation}
\begin{aligned}
\nabla_{\phi_{c}} \mathcal{E}(\phiv) &
 = \E_{ \piv\sim{\text{Dir}}(\mathbf{1}_{C})}[
 g_{\text{AR}}(\piv^{_{{c} \leftrightharpoons j}})_c]\\
 g_{\text{AR}}(\piv^{_{{c} \leftrightharpoons j}})_c:&=
 f( z^{_{{c} \leftrightharpoons j}} )(1-{C} \pi_j 
 ),\\
 z^{_{{c} \leftrightharpoons j}} :&=\argmin\nolimits_{i\in\{1,\ldots,{C}\}} \pi_i^{_{{c} \leftrightharpoons j}} e^{-\phi_i},
\end{aligned}\label{eq:Grad_AR_swap}
 \end{equation}
 where we have applied identity $\pi_{c}^{_{{c} \leftrightharpoons j}} = \pi_j $ and Property~IV. 
We refer to $z$ defined in \eqref{eq:Grad_AR} as the ``true action,'' and $z^{_{{c} \leftrightharpoons j}}$ defined in \eqref{eq:Grad_AR_swap} as the ${c}$th ``pseudo action'' given $j$ as the reference category. Note the pseudo actions satisfy the following properties: $z^{_{{c} \leftrightharpoons j}} =z^{_{{j} \leftrightharpoons c}} $ and $z^{_{{c} \leftrightharpoons j}} =z$ if ${{c}=j}$, and the number of unique values in $\{z^{_{{c} \leftrightharpoons j}} \}_{{c},j}$ that are different from the true action $z$ is between $0$ and ${C}-1$.

With \eqref{eq:CatGrad}, we have another useful property as

\emph{\textbf{Property V.}} ~~$\sum_{{c}=1}^{C}\nabla_{\phi_{c}} \mathcal{E}(\phiv) = 0$. 

Combining it with the 
estimator in 
\eqref{eq:Grad_AR_swap} leads to 
\begin{equation}
\begin{aligned}
&\textstyle\E_{ \piv\sim{\text{Dir}}(\mathbf{1}_{C})}\big[\frac{1}{{C}}\sum_{c=1}^{C}
g_{\text{AR}}(\piv^{_{{c} \leftrightharpoons j}})_c
 \big] =0.
\end{aligned}\label{eq:baseline}
 \end{equation} 
Thus we can utilize $\frac{1}{{C}}\sum_{c=1}^{C}
g_{\text{AR}}(\piv^{_{{c} \leftrightharpoons j}})_c$
as a baseline function that is nonzero in general but has zero expectation under $\piv\sim\mbox{Dir}(\mathbf{1}_{C})$. Subtracting \eqref{eq:baseline} from \eqref{eq:Grad_AR_swap} leads to another unbiased
 estimator, with category $j$ as the reference, as 
\begin{equation}
\begin{aligned}
&\nabla_{\phi_{c}}\mathcal{E}(\phiv) \textstyle = \E_{ \piv\sim{\text{Dir}}(\mathbf{1}_{C})} 
 [g_{{\text{ARS}}}(\piv,j)_c],
 \label{eq:ARM-cat}\\
&\!\!\!\! g_{{\text{ARS}}}(\piv,j)_{c} :=
g_{\text{AR}}(\piv^{_{{c} \leftrightharpoons j}})_c - \textstyle \frac{1}{{C}}\sum_{m=1}^{C}
g_{\text{AR}}(\piv^{_{{m} \leftrightharpoons j}})_m,\\
&\textstyle =\big[ f(z^{_{{c} \leftrightharpoons j}} )- \frac{1}{{C}}\sum_{m=1}^{C} f(z^{_{{m} \leftrightharpoons j}} )\big](1-C\pi_j),
\end{aligned}\!\!
 \end{equation}
which is referred to as the AR-swap (ARS) estimator, due to the use of variable-swapping in its derivation from AR. 

 \subsection{ARSM: Augment-REINFORCE-Swap-Merge}

For ARS in \eqref{eq:ARM-cat},
when 
 the reference category~$ j$ is randomly chosen from $\{1,\ldots,C\}$ and hence is independent of $\piv$ and~$\phiv$, 
it is unbiased. 
Furthermore, we find 
that it can be further improved, especially when $C$ is large, 
by adding a merge step to construct the ARS-merge (ARSM) estimator: 
 \begin{theorem}[
 ARSM estimator\label{theorem 3}] 
The gradient of $\mathcal{E}(\phiv)=\E_{z\sim\emph{\text{Cat}}(\sigma(\phiv))}[f{}(z)] $ 
with respect to $\phi_{c}$, can be expressed as 
\begin{equation}
\begin{aligned}
 &\nabla_{\phi_{c}}\mathcal{E}(\phiv) \textstyle= 
 \E_{ \piv\sim\emph{\text{Dir}}(\mathbf{1}_{C})} \big[ g_{\emph{\text{ARSM}}}(\piv)_{c} 
 \big],\\
&g_{\emph{\text{ARSM}}}(\piv)_{c} :\textstyle= \frac{1}{C}\sum_{j=1}^{C}
 g_{\emph{\text{ARS}}}(\piv,j)_{c}\\
 & =
\textstyle \sum_{j=1}^C
 \big[
 f(z^{_{{c} \leftrightharpoons j}} )- \frac{1}{{C}}\sum_{m=1}^{C} f(z^{_{{m} \leftrightharpoons j}} ) \big]
 (\frac{1}{C} -\pi_j). 
 \label{eq:ARM-cat2}
\end{aligned}
\end{equation}
 \end{theorem}

Note ARSM requires $C(C-1)/2$ swaps to generate pseudo actions, the unique number of which that differ from $z$ is between $0$ and $C-1$; a naive implementation requires $O(C^2)$ $\argmin$ operations, which, however, is totally unnecessary, as in general it can at least
be made below $O(2C)$ and hence is scalable even $C$ is very large ($e.g.$, $C=10,000$); please see Appendix \ref{sec:fast compute} and the provided code for more details. 
Note if all pseudo actions $z^{_{{c} \leftrightharpoons j}}$ are the same as the true action $z$, then the gradient estimates will be zeros for all $\phi_{c}$. 
\begin{cor}\label{cor:ARM}
When ${C}=2$, both the ARS estimator in \eqref{eq:ARM-cat} and ARSM estimator in \eqref{eq:ARM-cat2} 
reduce to the unbiased binary ARM estimator introduced in \citet{ARM_ICLR2019}.
\end{cor}

Detailed derivations and proofs 
 are provided in Appendix~\ref{sec:derivation}. Note for $C=2$, 
Proposition 4 of \citet{ARM_ICLR2019} shows that the ARM estimator is the AR estimator combined with an optimal baseline that is subject to an anti-symmetric constraint. 
When $C>2$, however, such type of theoretical analysis becomes very challenging for both the ARS and ARSM estimators. 
For example, it is even unclear how to define anti-symmetry for categorical variables. Thus in what follows we will focus on empirically evaluating the effectiveness of both ARS and ARSM for variance reduction.

\section{ARSM Estimator for Multivariate, Hierarchical, and Sequential Settings}
This section shows how 
the proposed univariate ARS and ARSM estimators can be generalized into multivariate, hierarchical, and sequential settings. We summarize ARS and ARSM (stochastic) gradient ascent for various types of categorical latent variables in Algorithms \ref{alg:ARSM}-\ref{alg:ARM} of the Appendix. 
\subsection{ARSM for Multivariate Categorical Variables and Stochastic Categorical Network}

We generalize the univariate AR/ARS/ARSM estimators to multivariate ones, which can backpropagate the gradient through a $K$ dimensional vector of $C$-way categorical variables as $\zv=(z_1,\ldots, z_K)$, where $z_k\in\{1,\ldots,{C}\}$. We further generalize them to backpropagate the gradient through multiple stochastic categorical layers, the $t$th layer of which consists of a $K_t$-dimensional $C$-way categorical vector as $\zv_t=(z_{t1},\ldots,z_{tK_t})'\in\{1,\ldots,C\}^{K_t}$.
We defer all the details to Appendix \ref{sec:derivation1} due to space constraint. 

Note for categorical variables, especially in multivariate and/or hierarchical settings, the ARS/ARSM estimators may appear fairly complicated due to their variable-swapping operations. Their implementations, however, are actually relatively straightforward, as shown in Algorithms \ref{alg:ARSM} and \ref{alg:ARSM1} of the Appendix, and the provided Python code. 
\subsection{ARSM for Discrete-Action Policy Optimization}
In RL with a discrete action space with $C$ possible actions, at time $t$, the agent with state $\sv_t$ chooses action $a_t\in\{1,\ldots,C\}$ according to policy 
$$\pi_{\thetav}(a_t\given \sv_t) := \mbox{Cat}(a_t; \sigma(\phiv_t)),~~\phiv_t:=\mathcal{T}_{\thetav}(\sv_t),$$ where 
$\mathcal{T}_{\thetav}(\cdotv)$ denotes a neural network parameterized by~$\thetav$;
the agent receives award $r(\sv_t,a_t)$ at time $t$, and state $\sv_t$ transits to state $\sv_{t+1}$ according to $\mathcal{P}(\sv_{t+1}\given \sv_t,a_t)$. 
With discount parameter $\gamma\in(0,1]$, policy gradient methods optimize $\thetav$ 
to maximize the expected reward $J(\thetav)=\E_{\mathcal{P},\pi_{\thetav}}[\sum_{t=0}^\infty \gamma^{t} r(\sv_{t},a_{t})]$ \citep{sutton2018reinforcement,sutton2000policy, peters2008natural,schulman2015trust}. With $Q(\sv_t,a_t):= \E_{\mathcal{P},\pi_{\thetav}}[\sum_{t'=t}^\infty \gamma^{t'-t} r(\sv_{t'},a_{t'})] $ denoted as the action-value functions, $\hat Q(\sv_t,a_t): =\sum_{t'=t}^\infty \gamma^{t'-t} r(\sv_{t'},a_{t'})$ as their sample estimates, and $\rho_{\pi}(\sv):= \sum_{t=0}^\infty \gamma^t \mathcal{P}(\sv_t=\sv \given \sv_0,\pi_{\thetav})$ as the unnormalized discounted state visitation frequency, the policy gradient via REINFORCE \citep{williams1992simple} can be expressed as
\begin{align}
 \textstyle \small \nabla_{\thetav} J(\thetav) \!=\!\E_{a_t\sim \pi_{\thetav}(a_t| \sv_t),\,\sv_t\sim \rho_{\pi}(\sv)}[\nabla_{\thetav} \ln \pi_{\thetav}(a_t | \sv_t)Q(\sv_t,a_t)].\notag
\end{align}
For variance reduction, one often subtracts state-dependent baselines $b(\sv_t)$ from $\hat Q(\sv_t,a_t)$ \citep{williams1992simple,greensmith2004variance}. In addition, several different action-dependent baselines $b(\sv_t,a_t)$ have been recently proposed \cite{gu2017q,grathwohl2017backpropagation,wu2018variance,liu2018actiondependent}, though their promise in appreciable variance reduction without introducing bias for 
policy gradient has been questioned by \citet{pmlr-v80-tucker18a}. 

Distinct from all previous baseline-based variance reduction methods, in this paper, we develop both the ARS and ARSM policy gradient estimators, which use the 
action-value functions ${Q}(\sv_t,a_t)$ themselves combined with pseudo actions to achieve variance reduction:

\begin{proposition}[ARS/ARSM policy gradient] 
\label{lem1}
The policy gradient $\nabla_{\thetav} J(\thetav)$ 
 can be expressed as
\ba{
\textstyle\small \nabla_{\thetav} J(\thetav)=
 \E_{{\varpiv_t\sim\emph{\text{Dir}}(\mathbf{1}_C)},~\sv_t\sim \rho_{\pi}(\sv)}\big[ \nabla_{\thetav}\sum_{c=1}^C g_{tc} \phi_{tc}\big],
 \label{eq: grad_aggregate}
}
where $\varpiv_t=(\varpi_{t1},\ldots,\varpi_{tC})'$ and $\phi_{tc}$ is the $c$th element of $\phiv_t=\mathcal{T}_{\thetav}(\sv_t) \in \mathbb{R}^C$; under the ARS estimator, we have 
\ba{
g_{tc}
:& \textstyle=
 f_{t\Delta}^{_{{c} \leftrightharpoons j_t}}(\varpiv_t)(1- C\varpi_{tj_t}) ,\notag\\ 
 f_{t\Delta}^{_{{c} \leftrightharpoons j_t}}(\varpiv_t): &= \textstyle
Q(\sv_t,a_t^{_{{c} \leftrightharpoons j_t}})- \frac{1}{C}\sum_{m=1}^C Q(\sv_t,a_t^{_{{m} \leftrightharpoons j_t}}),\notag\\
\textstyle a_t^{_{{c} \leftrightharpoons j_t}}: &\textstyle= \argmin_{i\in\{1,\ldots,C\}} \varpi_{ti}^{_{{c} \leftrightharpoons j_t}} e^{-\phi_{ti}},~~~~~~~~~~~~~~~\label{ARS_RL}
}
where $j_t\in\{1,\ldots,C\}$ is a randomly selected reference category for time step $t$;
under the ARSM estimator, we have 
\ba{\label{ARSM_RL}
&
 g_{tc}
 \textstyle :=\sum_{j=1}^Cf_{t\Delta}^{_{{c} \leftrightharpoons j}}(\varpiv_t)(\frac{1}{C}- \varpi_{tj}) .
}
\end{proposition}

Note as the number of unique actions among $a_t^{_{{m} \leftrightharpoons j}}$ is as few as one, in which case the ARS/ARSM gradient is zero and there is no need at all to estimate the $Q$ function, and as many as $C$, in which case one needs to estimate the $Q$ function $C$ times.
Thus if the computation of estimating $Q$ once is $O(1)$, then the worst computation for an episode that lasts $T$ time steps before termination is $O(TC)$. {Usually the number of distinct pseudo actions will decrease dramatically as the training progresses. We illustrate this in Figure \ref{fig:entropy}, where we show the trace of categorical variable's entropy and number of distinct pseudo actions that differ from the true action.}
 Examining \eqref{ARS_RL} and \eqref{ARSM_RL} shows that the ARS/ARSM policy gradient estimator can be intuitively understood as a ``try-and-see self-critic'' method, which eliminates the need of constructing baselines and estimating their parameters for variance reduction. To decide the gradient direction of whether increasing the probability of action $c$ at a given state, it compares the pseudo-action reward $ Q(\sv_t,a_t^{_{{c} \leftrightharpoons j}})$ with the average of all pseudo-action rewards $\{ Q(\sv_t,a_t^{_{{m} \leftrightharpoons j}})\}_{m=1,C}$. 
If the current policy is very confident on taking action $a_t$ at state $\sv_t$, which means $\phi_{ta_t}$ dominates the other $C-1$ elements of $\phiv_t=\mathcal{T}_{\thetav}(\sv_t)$, then it is very likely that $a_t^{_{{m} \leftrightharpoons j_t}}=a_t$ for all $m$, which will lead to zero gradient at time $t$. On the contrary, if the current policy is uncertain about which action to choose, then more pseudo actions that are different from the true action are likely to be generated. 
This mechanism encourages exploration when the policy is uncertain, and balance the tradeoff of exploration and exploitation intrinsically. It also explains our empirical observations that ARS/ARSM tends to generate a large number of unique pseudo actions in the early stages of training, leading to fast convergence, and significantly reduced number 
once the policy becomes sufficiently certain, leading to stable performance after convergence. 

\begin{figure*}[th]
 \centering
 \includegraphics[width=0.89\textwidth,height=6.2cm]{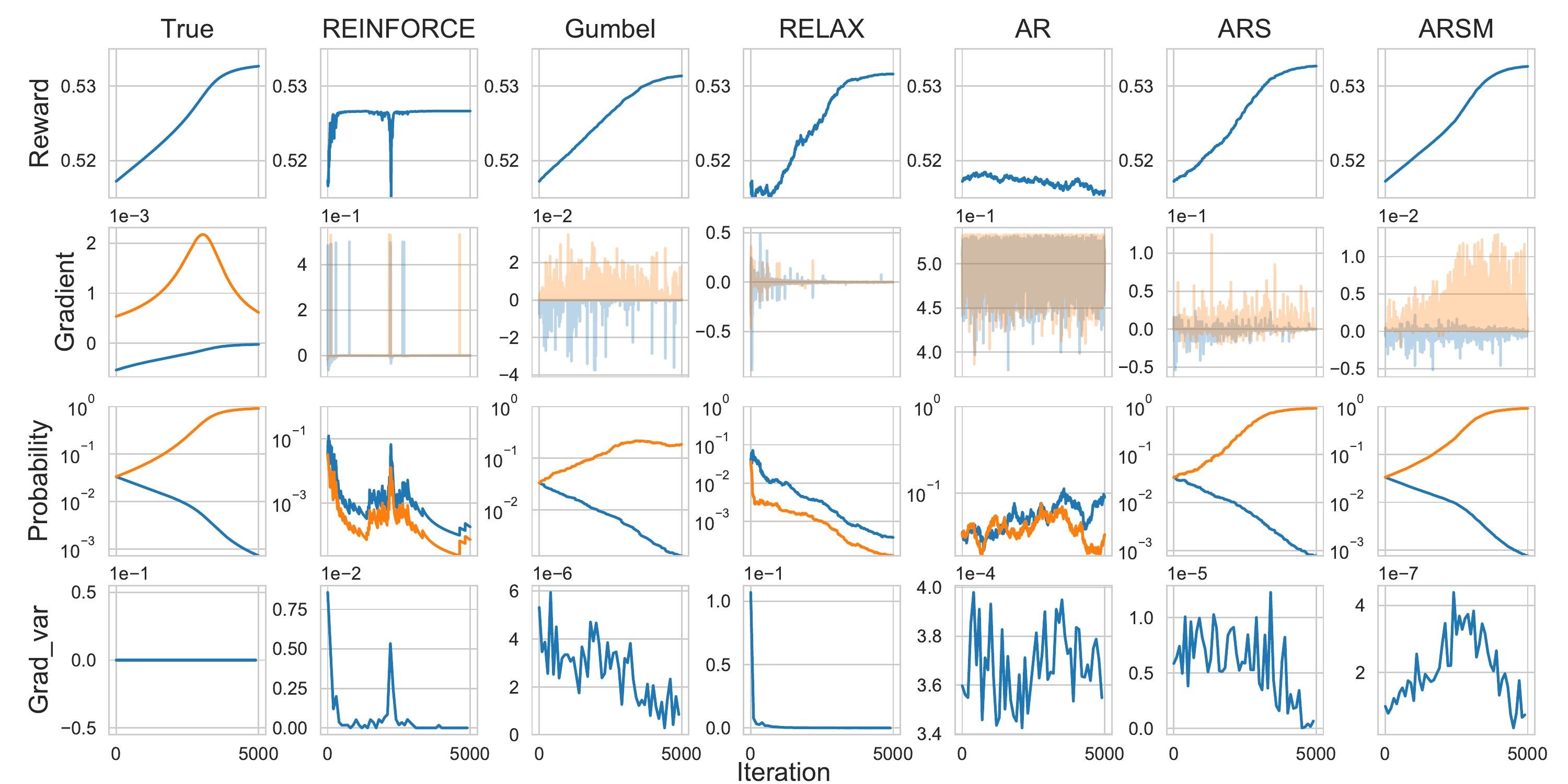}\vspace{-4.5mm}
 \caption{Comparison of a variety of gradient estimators in maximizing \eqref{eq:toy}. The optimal solution is $\sigma(\phiv) = (0, \ldots, 1)$, which means $z = C$ with probability one. The reward is computed analytically by $\E_{z\sim\text{Cat}(\sigma(\phiv))}[f(z)]$ with maximum as $0.533$. Rows 1, 2, and 3 show the trace plots of reward $\E[f(z)]$, the gradients with respect to $\phi_1$ and $\phi_C$, and the probabilities $\sigma(\phiv)_1$ and $ \sigma(\phiv)_C$, respectively. Row 4 shows the gradient variance estimation with 100 Monte Carlo samples at each iteration, averaged over categories $1$ to $C$.}\vspace{-3mm}
 \label{fig:toy}
\end{figure*}

\section{Experimental Results}

In this section, we use a toy example for illustration, demonstrate both multivariate and hierarchical
settings with categorical latent variable models, and demonstrate the sequential setting with discrete-action policy optimization. Comparison of gradient variance between various algorithms can be found in Figures \ref{fig:toy} and \ref{fig:rl}-\ref{fig:vae_var}.

 \subsection{Example Results on Toy Data}

 To illustrate the working mechanism of the ARSM estimator, we consider learning $\phiv\in\mathbb{R}^C$ to maximize 
 \ba{
\textstyle &\E_{z\sim\text{Cat}(\sigma(\phiv))}[f(z)], ~~f(z):=0.5+z/(CR) , \label{eq:toy}
 }
 where $z\in\{1,\ldots,C\}$. The optimal solution is $\sigma(\phiv)=(0,\ldots,0,1)$, which leads to the maximum expected reward of $0.5+1/R$. 
The larger the $C$ and/or $R$ are, the more challenging the optimization becomes. 
 We first set $C=R=30$ that are small enough to allow existing algorithms to perform reasonably well. Further increasing $C$ or $R$ will often fail existing algorithms and ARS, while ARSM always performs almost as good as the true gradient when used in optimization via gradient ascent. We include the results for $C = 1,000$ and $10,000$ in Figures \ref{fig:toy1000} and \ref{fig:toy10000} of the Appendix.
 
 We perform an ablation study of the proposed AR, ARS, and ARSM estimators. We also make comparison to two representative low-variance estimators, including the biased Gumbel-Softmax estimator \citep{jang2016categorical,maddison2016concrete} that applies the reparameterization trick after continuous relaxation of categorical variables, and the unbiased RELAX estimator of \citet{grathwohl2017backpropagation} that combines reparameterization and REINFORCE with an adaptively estimated baseline. 
 We compare them in terms of the expected reward as $\sum_{c=1}^C \sigma(\phiv)_cf(c)$, gradients for $\phi_c$, probabilities $\sigma(\phiv)_c$, and gradient variance. Note when $C=2$, both ARS and ARSM reduce to the ARM estimator, which has been shown in \citet{ARM_ICLR2019} to outperform a wide variety of estimators for binary variables, including the REBAR estimator of \citet{tucker2017rebar}. 
The true gradient in this example can be computed analytically as in~\eqref{eq:CatGrad}. All estimators in comparison use a single Monte Carlo sample for gradient estimation. We initialize $\phi_c=0$ for all $c$ and fix the gradient-ascent stepsize as one. 

As shown in Figure \ref{fig:toy}, without appropriate variance reduction, both AR and REINFORCE either fail to converge or converge to a low-reward solution. We notice RELAX for $C=R=30$ is not that stable across different runs; 
in this particular run, it manages to obtain a relatively high reward, but its probabilities converge towards a solution that is different from the optimum $\sigma(\phiv)=(0,\ldots,0,1)$. By contrast, Gumbel-Softmax, ARS, and ARSM all robustly reach probabilities close to the optimum $\sigma(\phiv)=(0,\ldots,0,1)$ after 5000 iterations across all random trials. The gradient variance of ARSM is about one to four magnitudes less than these of the other estimators, which helps explain why ARSM is almost identical to the true gradient 
in moving $\sigma(\phiv)$ towards the optimum that maximizes the expected reward. 
The advantages of ARSM become even clearer in more complex settings where analytic gradients become intractable to compute, as shown below.

 \subsection{Categorical Variational Auto-Encoders}

For optimization involving expectations with respect to multivariate categorical variables, we consider a variational auto-encoder (VAE) with a single categorical stochastic hidden layer. We further consider a categorical VAE with two categorical stochastic hidden layers to illustrate optimization involving expectations with respect to hierarchical multivariate categorical variables.

Following \citet{jang2016categorical}, we consider a VAE with a categorical hidden layer to model $D$-dimensional binary observations. The decoder parameterized by $\thetav$ is expressed as
$
 p_{\thetav}(\xv \given \zv) = \prod_{i=1}^D p_{\thetav}(x_i \given \zv) $, where $\zv\in\{1,\ldots,C\}^K$ is a $K$-dimensional $C$-way categorical vector
and $p_{\thetav}(x_i \given \zv)$ is Bernoulli distributed. The encoder parameterized by $\phiv$ is expressed as $q_{\phiv}(\zv \given \xv) \textstyle = \prod_{k=1}^K q_{\phiv}(z_k \given \xv)$. We set the prior as $p(z_k=c)=1/C$ for all $c$ and $k$.
For optimization, we maximize the evidence lower bound (ELBO) as 
\ba{
\mathcal{L}(\xv) \textstyle = \bE_{\zv \sim q_{\phiv}(\zv\given \xv)}\big[ \ln\ \frac{p_{\thetav}(\xv\given \zv)p(\zv)}{q_{\phiv}(\zv\given \xv)}\big]. 
} 

We also consider a two-categorical-hidden-layer VAE, whose encoder and decoder are constructed as
\bas{
&q_{\phiv_{1:2}}(\zv_1,\zv_2 \given \xv) \textstyle = q_{\phiv_1}(\zv_1 \given \xv)q_{\phiv_{2}}(\zv_2 \given \zv_1), \\
&p_{\thetav_{1:2}}(\xv \given \zv_1,\zv_2) \textstyle = p_{\thetav_{1}}(\xv\given \zv_1)p_{\thetav_{2}}(\zv_1 \given \zv_2), 
}
where $\zv_1,\zv_2\in\{1,\ldots,C\}^K$. The 
 ELBO is expressed as
\ba{\small
\mathcal{L}(\xv) \textstyle = \bE_{q_{\phiv_{1:2}}(\zv_1,\zv_2 \given \xv)} \Big[\ln \frac{p_{\thetav_{1}}(\xv\given \zv_1)p_{\thetav_{2}}(\zv_1\given \zv_2) p(\zv_2)}{q_{\phiv_{1}}(\zv_1 \given \xv)q_{\phiv_{2}}(\zv_2 \given \zv_1)}\Big].
}

\begin{table*}[th]
\small
 \caption{\small Comparison of training and testing negative ELBOs (nats) on binarized MNIST between ARSM and various gradient estimators.
 }\label{tab:vae}
 \centering
{
 \begin{tabular}{c@{\hskip8pt}cc@{\hskip7pt}c@{\hskip7pt}c@{\hskip7pt}c@{\hskip7pt}c||@{\hskip7pt}c@{\hskip7pt}cc}
 \toprule
 Gradient estimator & REINFORCE & RELAX & ST Gumbel-S. & AR & ARS & ARSM & Gumbel-S.-2layer & ARSM-2layer\\ 
 \midrule
$ - $ELBO (Training) & 127.0 & 117.4 & 94.1 & 133.6 & 97.4 & {82.0} & 91.3 & \textbf{78.3}\\ 
 \midrule
$- $ELBO (Testing) & 127.6 & 118.7 & 96.4 & 135.0 & 101.4 & \textbf{86.7} & 98.3 & {89.5} \\ 
 \bottomrule
 \end{tabular}}
 \vspace{-4mm}
\end{table*}
 \begin{figure}[th]
 \centering
 \includegraphics[width=0.44\textwidth,height=4.5cm]{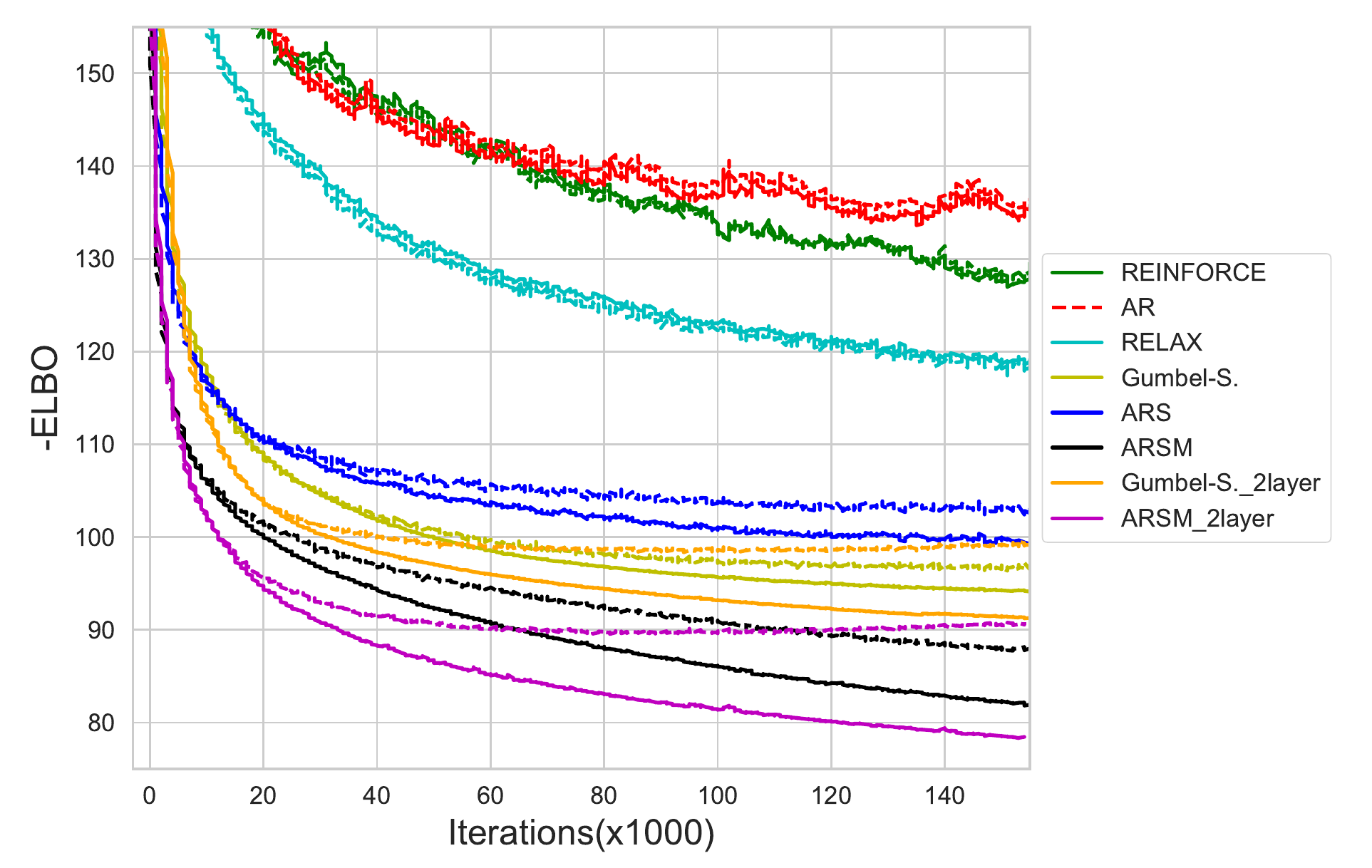} \vspace{-4.5mm}
 \caption{Plots of negative ELBOs (nats) on binarized MNIST against training iterations (analogous ones against times are shown in Figure \ref{fig:vae_time}). The solid and dash lines correspond to the training and validation, respectively (best viewed in color).}
 \label{fig:vae}
 \vspace{-2mm}
\end{figure}

For both categorical VAEs, we set $K = 20$ and $C = 10$. We train them on a binarized MNIST dataset as in \citet{van2017neural} by thresholding each pixel value at 0.5. 
Implementations of the VAEs with one and two categorical hidden layers are summarized in Algorithms \ref{alg:ARSM} and \ref{alg:ARSM1}, respectively; see the provided code for more details.

We consider the AR, ARS, and ARSM estimators, and include the REINFORCE \citep{williams1992simple}, Gumbel-Softmax \citep{jang2016categorical}, and RELAX \citep{grathwohl2017backpropagation} estimators for comparison. We note that \citet{jang2016categorical} has already shown Gumbel-Softmax outperforms a wide variety of previously proposed estimators; see \citet{jang2016categorical} and the references therein for more details.

We present the trace plots of the training and validation negative ELBOs in Figure~\ref{fig:vae} and gradient variance in Figure~\ref{fig:vae_var}. The numerical values are summarized in Table \ref{tab:vae}. We use 
the Gumbel-Softmax code \footnote{\href{https://github.com/ericjang/gumbel-softmax}{https://github.com/ericjang/gumbel-softmax}} to obtain the results of the VAE with a single categorical hidden layer, and modify it with our best effort for the VAE with two categorical hidden layers; we modify the RELAX code \footnote{\href{https://github.com/duvenaud/relax}{https://github.com/duvenaud/relax}} with our best effort to allow it to optimize VAE with a single categorical hidden layer. 
For the single-hidden-layer VAE, we connect its latent categorical layer $\zv$ and observation layer $\xv$ with two nonlinear deterministic layers; for the two-hidden-layer VAE, we add an additional categorical hidden layer $\zv_2$ that is linearly connected to the first one. See Table \ref{tab:Network} of the Appendix for detailed network architectures. In our experiments, all methods use exactly the same network architectures and data, set the mini-batch size as 200, and are trained by the Adam optimizer \cite{kingma2014adam}, whose learning rate is selected from $\{1,2,\ldots,5\}\times 10^{-4}$ 
using the validation set. We notice for the same model, a large learning rate can result in reduced training loss but increased testing loss, which suggests overfitting.
 
The results in Table \ref{tab:vae} and Figure~\ref{fig:vae} clearly show that for optimizing the single-categorical-hidden-layer VAE, {both ARS and ARSM estimators 
perform well in terms of} both training and testing ELBOs. In particular, ARSM outperforms all the other estimators by a large margin. We also consider Gumbel-Softmax by computing its gradient with 25 Monte Carlo samples, making it run as fast as the provided ARSM code does per iteration. In this case, both algorithms take similar time but ARSM achieves $-$ELBOs for the training and testing sets as $82.0$ and $86.7$, respectively, while those of Gumbel-Softmax are $93.6$ and $95.9$, respectively.  The performance gain of ARSM can be explained by both its unbiasedness and a clearly lower variance exhibited by its gradient estimates in comparison to all the other estimators, as shown in Figure \ref{fig:vae_var} of the Appendix. The results on the two-categorical-hidden-layer VAE, which adds a linear categorical layer on top of the single-categorical-hidden-layer VAE, {also suggest that ARSM outperforms the biased Gumbel-Softmax estimator.}

\subsection{Maximum Likelihood Estimation for a Stochastic Categorical Network}

Denoting $\xv_l,\xv_u\in\mathbb{R}^{392}$ 
as the lower and upper halves of an MNIST digit, respectively, 
we consider a standard benchmark task of estimating the conditional distribution $p_{\thetav_{0:2}}(\xv_l\given \xv_u)$ \citep{raiko2014techniques,bengio2013estimating,gu2015muprop,jang2016categorical,tucker2017rebar}. We consider
a stochastic categorical network with two stochastic categorical hidden layers, expressed as
\bas{
\xv_l&\sim \mbox{Bernoulli}( \sigma(\mathcal{T}_{\thetav_0} (\bv_1))),\\
\bv_1&\textstyle \sim \prod_{c=1}^{20} \mbox{Cat}(b_{1c}; \sigma(\mathcal{T}_{\thetav_1} (\bv_2)_{[10(c-1)+(1:10)]})),\\
\bv_2&\textstyle \sim \prod_{c=1}^{20} \mbox{Cat}(b_{2c}; \sigma(\mathcal{T}_{\thetav_2} (\xv_u)_{[10(c-1)+(1:10)]})),
}
where both $\bv_1$ and $\bv_2$ are 20-dimensional 10-way categorical variables, $\mathcal{T}_{\thetav}(\cdotv)$ denotes linear transform, $\mathcal{T}_{\thetav_2} (\xv_u)_{[10(c-1)+(1:10)]}$ is a 10-dimensional vector consisting of elements $10(c-1)+1$ to $10c$ of $\mathcal{T}_{\thetav_2} (\xv_u)\in\mathbb{R}^{200}$, $\mathcal{T}_{\thetav_1} (\bv_2)\in\mathbb{R}^{200}$, and $\mathcal{T}_{\thetav_0}\in\mathbb{R}^{392}$. Thus we can consider the network structure as 392-200-200-392, making the results directly comparable with these in \citet {jang2016categorical} for stochastic categorical network. We approximate 
 $\log p_{\thetav_{0:2}}(\xv_l\given \xv_u)$ 
 with $K$ Monte Carlo samples as \ba{\textstyle \log \frac{1}{K}\sum_{k=1}^K 
 \mbox{Bernoulli}(\xv_l; \sigma(\mathcal{T}_{\thetav_0} ( \bv_1^{(k)}))),\label{eq:ML}} 
 where 
$ \bv_1^{(k)}\textstyle \sim \prod_{c=1}^{20} \mbox{Cat}(b_{1c}^{(k)}; \sigma(\mathcal{T}_{\thetav_1} (\bv_2^{(k)})_{[10(c-1)+(1:10)]}))$, 
$\bv_2^{(k)}\textstyle \sim \prod_{c=1}^{20} \mbox{Cat}(b_{2c}^{(k)}; \sigma(\mathcal{T}_{\thetav_2} (\xv_u)_{[10(c-1)+(1:10)]}))$. We perform training with $K=1$, which can also be considered as optimizing on a single-Monte-Carlo-sample estimate of the lower bound of the log marginal likelihood. 
We use Adam \citep{kingma2014adam}, with the learning rate set as $10^{-4}$, mini-batch size as 100, and number of training epochs as 2000. 
 Given the inferred point estimate of $\thetav_{0:2}$, 
 we evaluate the accuracy of conditional density estimation by estimating the negative log-likelihood $-\log p_{\thetav_{0:2}}(\xv_l\given \xv_u)$ using \eqref{eq:ML}, averaging over the test set with $K=1000$.

 \begin{table}[t]
\small\vspace{-2.5mm}
 \caption{\small Comparison of the test negative log-likelihoods between ARSM and various gradient estimators in \citet{jang2016categorical}, for the MNIST conditional distribution estimation benchmark task. }\label{tab:SBN}
 \centering
{\small
 \begin{tabular}{c@{\hskip7pt}cc@{\hskip6pt}c@{\hskip6pt}c@{\hskip6pt}}
 \toprule
 Gradient estimator & ARSM & ST & Gumbel-S. & MuProp \\
 \midrule
 $-\log p(\xv_l\given \xv_u)$ & \textbf{58.3 $\pm$ 0.2} & 61.8 & 59.7 & 63.0 \\
 \bottomrule
 \end{tabular}}\vspace{-4mm}
\end{table}
 
 As shown in Table \ref{tab:SBN}, optimizing a stochastic categorical network with the ARSM estimator 
 achieves the lowest test negative log-likelihood, outperforming all previously proposed 
 gradient estimators on the same structured stochastic networks, including straight through (ST) \citep{bengio2013estimating} and ST Gumbel-Softmax \citep{jang2016categorical} that are biased, and MuProp \citep{gu2015muprop} that is unbiased.

 \subsection{Discrete-Action Policy Optimization}

 \begin{figure*}[t]
 \centering
 \includegraphics[width=0.72\textwidth,height=5.5cm]{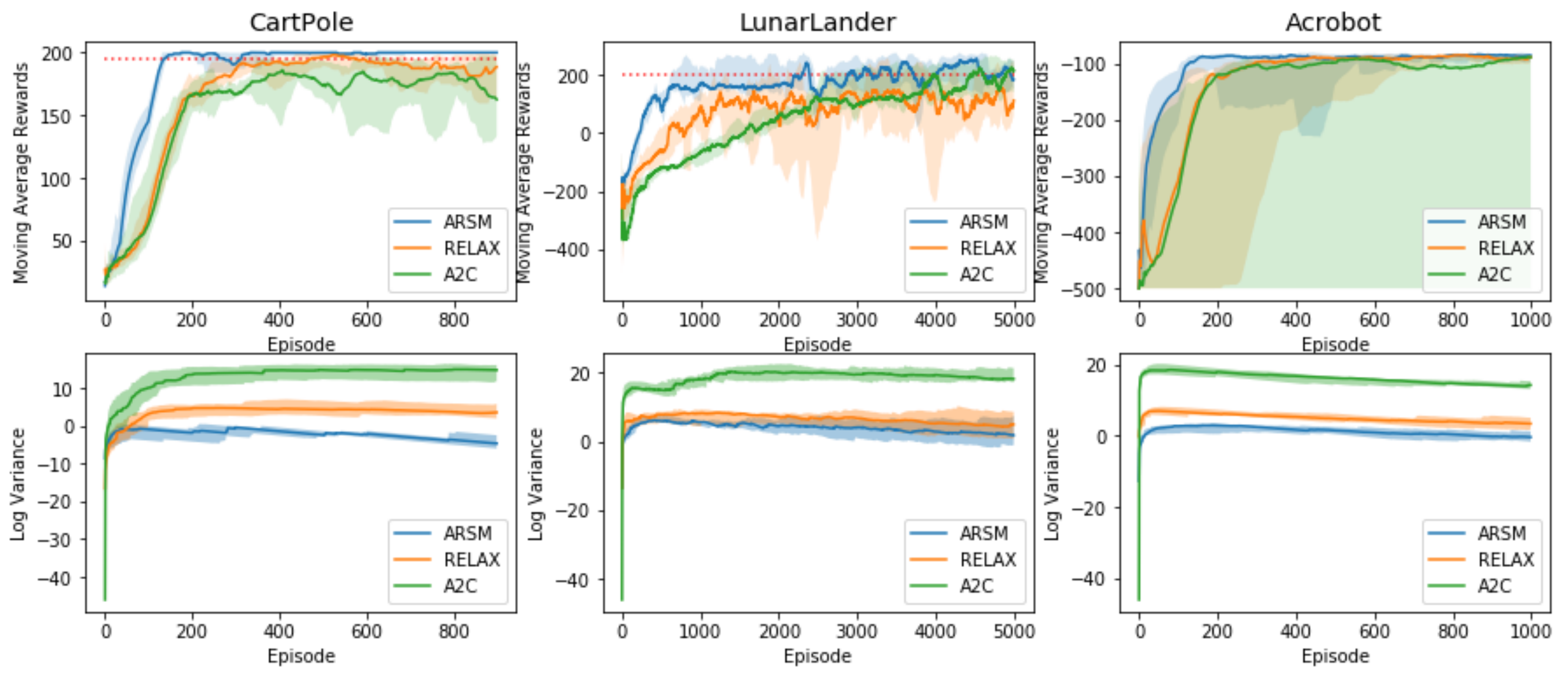}
 \vspace{-4.5mm}
 \caption{\textsl{Top row:} Moving average reward curves. \textsl{Bottom row:} Log-variance of gradient estimator. In each plot, the solid lines are the median value of ten  independent runs (ten different random seeds for random initializations). The opaque bars are $10$th and $90$th percentiles. Dashed straight lines in Cart Pole and Lunar Lander represent task-completion criteria.}
 \label{fig:rl} \vspace{-3mm}
\end{figure*}

 The key of applying the ARSM policy gradient shown in \eqref{ARSM_RL} is to provide, under the current policy $\pi_{\thetav}$, the action-value functions' sample estimates $\hat Q(\sv_t,a_t): =\sum_{t'=t}^\infty \gamma^{t'-t} r(\sv_{t'},a_{t'})$ 
 for all unique values in $\{a_t^{_{{c} \leftrightharpoons j}}\}_{c,j} $. 
 Thus ARSM is somewhat related to the \textit{vine} method proposed in \citet{schulman2015trust}, 
 which defines a heuristic rollout policy that chooses a subset of the states along the true trajectory as the ``rollout set,'' samples $K$ pseudo actions uniformly at random from the discrete-action set at each state of the rollout set, and performs a single rollout for each state-pseudo-action-pair to estimate its action-value function~$Q$. ARSM chooses its rollout set in the same manner, but is distinct from the \textit{vine} method in having a rigorously derived rollout policy: it swaps the elements of $ \varpiv_{t}\sim\mbox{Dir}(\mathbf{1}_C)$ to generate pseudo actions if state $\sv_t$ belongs to the rollout set; the number of unique pseudo actions that are different from the true action $a_t$ is a random number, which is positively related to the uncertainty of the policy and hence often negatively related to its convergence; and a single rollout is then performed for each of these unique pseudo actions to estimate its~$ Q$. 
 
As ARSM requires the estimation of $Q$ function for each unique state-pseudo-action pair using Monte Carlo rollout, 
 it could have high computational complexity if (1) the number of unique pseudo actions is large, and (2) each rollout takes many expensive steps (interactions with the environments) before termination. %
However, there exist ready solutions and many potential ones. As given a true trajectory, all the state-pseudo-action rollouts of ARSM can be independently simulated and hence all pseudo-action related $Q$'s can be estimated in an embarrassingly parallel manner.
Furthermore, in addition to Monte Carlo estimation, we can potentially adapt for ARSM a wide variety of off-the-shelf action-value function estimation methods \citep{sutton2018reinforcement}, to either accelerate the estimation of $ Q$ or further reduce the variance (though possibly at the expense of introducing bias). 
 In our experiment, for simplicity and clarity, 
 we choose to use Monte Carlo estimation to obtain $\hat Q$ for both the true trajectory and all state-pseudo-action rollouts. 
 The results for RELAX and A2C are obtained by running the code provided by 
 \citet{grathwohl2017backpropagation}\footnote{\href{https://github.com/wgrathwohl/BackpropThroughTheVoidRL}{https://github.com/wgrathwohl/BackpropThroughTheVoidRL}}.

 We apply the ARSM policy gradient to three representative RL tasks with discrete actions, including the Cart Pole, Acrobot, and Lunar Lander environments provided by OpenAI Gym \citep{brockman2016openai}, and compare it with advantage actor-critic algorithm (A2C) \citep{sutton2000policy} and RELAX \citep{grathwohl2017backpropagation}. We report the moving-average rewards 
 and the estimated log-variance of the gradient estimator at every episode; for each episode, the reward score is obtained by running the updated policy on a new random environment; and the variance is obtained by first applying
exponential moving averages to the first and second moments of each neural network parameter with decay $0.99$, and then taking the average of the estimated variances of all neural network parameters. 
 
Shown in Figure \ref{fig:rl} are the mean rewards over the last 100 steps; the opaque bar indicates $10$th and $90$th percentiles obtained by ten independent runs for each method (using $10$ different random seeds for random initializations); the solid line is the median value of these ten independent runs. ARSM outperform both baselines in all three tasks in terms of stability, moving average rewards, and log-variance of gradient estimator. All methods are cross validated by optimizers \{Adam Optimizer, RMSProp Optimizer\} and learning rates $\{1, 3, 10, 30\}\times 10^{-3}$. 
Both the policy and critic networks for A2C and RELAX have two 10-unit hidden layers with ReLU activation functions \citep{nair2010rectified}. The discount factor $\gamma$ is $0.99$ and entropy term is $0.01$. The policy network of ARSM is the same as that of A2C and RELAX, and 
the maximum number of allowed state-pseudo-action rollouts of ARSM 
is set as $16$, $64$, and $1024$ for Cart Pole, Acrobot, and Lunar Lander, respectively; see Algorithm \ref{alg:ARM} and the provided code for more details.  Using our current implementation that has not been optimized to fully take the advantage of parallel computing, to finish the number of episodes as in Figure \ref{fig:rl}, ARSM on average takes 677, 425, and 19050 seconds for CartPole, Acrobot, and LunarLander, respectively. For comparison,  for these three tasks, RELAX on average takes 139, 172, and 3493 seconds and A2C on average takes 92, 120, and 2708 seconds.

\section{Conclusion}
 
To backpropagate the gradients through categorical stochastic layers, 
we propose the augment-REINFORCE-swap-merge (ARSM) estimator that is unbiased and exhibits low variance. 
The performance of ARSM is almost identical to that of the true gradient when used for optimization involving a $C$-way categorical variable, even when $C$ is very large (such as $C=10,000$). For multiple $C$-way categorical variables organized into a single stochastic layer, multiple stochastic layers, or a sequential setting, the ARSM estimator clearly outperforms state-of-the-art methods, as shown in our experimental results for both categorical latent variable models and discrete-action policy optimization. 
We attribute the outstanding performance of ARSM to both its unbiasedness and its ability to control variance by simply combing its reward function with randomly generated pseudo actions, where the number of unique pseudo actions is positively related to the uncertainties of categorical distributions and hence negatively correlated to how well the optimization algorithm has converged; there is no more need to construct separate baselines and estimate their parameters, which also help make the optimization more robust. Some natural extensions of the proposed ARSM
estimator include applying it to reinforcement
learning with high-dimensional discrete-action spaces or multiple discrete-action agents, and various tasks in natural language processing such as sentence generation and machine translation.

\section*{Acknowledgements}
This research was supported in part by 
 Award IIS-1812699 from the U.S. National Science Foundation
and the McCombs Research Excellence Grant. The authors acknowledge the support
of NVIDIA Corporation with the donation of the Titan Xp GPU used for this research, and the
computational support of Texas Advanced Computing Center.

\bibliographystyle{icml2019}
\bibliography{reference2.bib,References052016_2.bib,References052016.bib,reference.bib}

\newpage
\appendix

\onecolumn

\begin{center}{\large{\textbf{ARSM: Augment-REINFORCE-Swap-Merge Gradient for Categorical Variables\\ \vspace{2mm} Supplementary Material}}}\end{center}

\section{Derivation of AR, ARS, and ARSM}\label{sec:derivation}
\subsection{Augmentation of a Categorical Variable}\label{sec:aug}

Let us denote $\tau\sim\mbox{Exp}(\lambda)$ as the exponential distribution, with probability density function $p(\tau\given \lambda)=\lambda e^{-\lambda \tau}$, where $\lambda>0$ and $\tau>0$. Its mean and variance are $\E[\tau] =\lambda^{-1} $ and $\mbox{var}[\tau]=\lambda^{-2}$, respectively. It is well known that, e.g. in \citet {Ross:2006:IPM:1197141}, if $\tau_i\sim\mbox{Exp}(\lambda_i)$ are independent exponential random variables for $i=1,\ldots,{C}$, then the probability that $\tau_z$, where $ {z\in\{1,\ldots,{C}\}}$, is the smallest can be expressed as
\ba{
\textstyle P\big(z = \argmin\nolimits_{{i\in\{1,\ldots,{C}\}}} \tau_i\big) = P\left(\tau_z< \tau_i,~ \forall~ i\neq z\right) = 
\frac{\lambda_z}{\sum\nolimits_{i=1}^{C} \lambda_i}~~. \label{eq:ExpRace}
}
Note this property, referred to as ``exponential racing'' in \citet{Lomax_2018}, is closely related to the Gumbel distribution (also known as Type-I extreme-value distribution) based latent-utility-maximization representation of multinomial logistic regression \citep{McFadden74,book_train_2009}, as well as the Gumbel-softmax trick \citep{maddison2016concrete,jang2016categorical}. This is because the exponential random variable $\tau\sim\mbox{Exp}(\lambda)$ can be reparameterized as $\tau = \epsilon/\lambda,~\epsilon\sim\mbox{Exp}(1) $, where $\epsilon\sim\mbox{Exp}(1) $ can be equivalently generated as $\epsilon=-\log u,~u\sim \mbox{Uniform}(0,1)$, and hence 
we have 
$$\argmin\nolimits_i \tau_i \stackrel{d} = \argmin\nolimits_i \{-\log u_i/\lambda_i\} = \argmax\nolimits_i \{\log \lambda_i - \log(-\log u_i)\} , $$
where $\tau_i\sim\mbox{Exp}(\lambda_i)$, ``$\stackrel{d} =$'' denotes ``equal in distribution,'' and $u_i\stackrel{iid}\sim \mbox{Uniform}(0,1)$; note that if $u\sim \mbox{Uniform}(0,1)$, then $- \log(-\log u)$ follows the Gumbel distribution \citep{book_train_2009}.

From \eqref{eq:ExpRace} we know that if
\ba{
z=\argmin\nolimits_{i\in\{1,\ldots,{C}\}} \tau_i~,\text{where }\tau_i \sim \mbox{Exp}(e^{\phi_i}),
\label{eq:BerRepara}
}
then $P(z \given \phiv) = e^{\phi_z}/\sum_{i=1}^{C} e^{\phi_i}$, and hence \eqref{eq:BerRepara} is an augmented representation of the categorical distribution $z\sim\text{Cat}(\sigma(\phiv))$; one may consider $\tau_i\sim \mbox{Exp}(e^{\phi_i})$ as augmented latent variables, the marginalization of which from $z=\argmin\nolimits_{i\in\{1,\ldots,{C}\}} \tau_i$ leads to $P(z \given \phiv)$. 
Consequently, the expectation with respect to the categorical variable of $C$ categories can be rewritten as one with respect to 
$C$ augmented exponential random variables as
\ba{
\mathcal{E}(\phiv)=\E_{z\sim\text{Cat}(\sigma(\phiv))}[f{}(z)] =
\E_{\tau_1\sim \text{Exp}(e^{\phi_1}), \ldots,\tau_{C}\sim \text{Exp}(e^{\phi_{C}})}[f{}(\argmin\nolimits_{i} \tau_i)].
\label{eq:E_z_repara}
}
Since the exponential random variable $\tau\sim\mbox{Exp}(e^{\phi})$ can be reparameterized as $\tau= \epsilon e^{-\phi},~\epsilon\sim\mbox{Exp}(1) $, we also have
\ba{
\mathcal{E}(\phiv)= \E_{ \epsilon_1,\ldots, \epsilon_{C}\,\stackrel {iid}\sim\, \text{Exp}(1)}[f{}(\argmin\nolimits_{i} \epsilon_i e^{-\phi_i} )]. 
\label{eq:E_z_repara1}
}
Note as the $\argmin$ operator
 is non-differentiable, the widely used reparameterization trick \citep{kingma2013auto,rezende2014stochastic} 
 is not applicable to computing the gradient of $\mathcal{E}(\phiv) $ via the reparameterized representation in 
\eqref{eq:E_z_repara1}.

\subsection{REINFORCE Estimator in the Augmented Space}
Using REINFORCE \citep{williams1992simple} on \eqref{eq:E_z_repara}, we have 
 $\nabla_{\phiv}\mathcal{E}(\phiv) = [\nabla_{\phi_1} \mathcal{E}(\phiv),\ldots,\nabla_{\phi_{C}} \mathcal{E}(\phiv)]'$, where
\ba{
\nabla_{\phi_{c}} \mathcal{E}(\phiv)&= \textstyle 
\E_{\tau_1\sim \text{Exp}(e^{\phi_1}), \ldots,\tau_{C}\sim \text{Exp}(e^{\phi_{C}})}\Big[f{}(\argmin\nolimits_{i} \tau_i) \nabla_{\phi_{c}} \log \prod_{i=1}^{C} \mbox{Exp}(\tau_i;e^{\phi_i}) \Big]\notag\\
 &=\E_{\tau_1\sim \text{Exp}(e^{\phi_1}), \ldots,\tau_{C}\sim \text{Exp}(e^{\phi_{C}})}[f{}(\argmin\nolimits_{i} \tau_i) \nabla_{\phi_{c}} \log \mbox{Exp}(\tau_{c};e^{\phi_{c}}) ]\notag\\
& = \E_{\tau_1\sim \text{Exp}(e^{\phi_1}), \ldots,\tau_{C}\sim \text{Exp}(e^{\phi_{C}})}[f{}(\argmin\nolimits_{i}\tau_i) (1-\tau_{c}e^{\phi_{c}})]. 
 \label{eq:Grad_1}
}

Below we show how to merge $\nabla_{\phi_{c}} \mathcal{E}(\phiv) $ and $- \nabla_{\phi_j} \mathcal{E}(\phiv) $ by first re-expressing \eqref{eq:Grad_1} into an expectation with respect to $iid$ exponential random variables, swapping the indices of these random variables, and then sharing common random numbers \citep{mcbook} 
to well control the variance of Monte Carlo integration.

\subsection{Merge of Augment-REINFORCE Gradients}\label{sec:merge}

A key observation of the paper is we can 
 re-express the expectation 
 in 
 \eqref{eq:Grad_1} as
\ba{
\nabla_{\phi_{c}} \mathcal{E}(\phiv)
& = \E_{ \epsilon_1,\ldots, \epsilon_{C}\,\stackrel {iid}\sim\, \text{Exp}(1)}[f{}(\argmin\nolimits_{i} \epsilon_i e^{-\phi_i} ) (1-\epsilon_{c})] }
Furthermore, we note that $\mbox{Exp}(1) \stackrel{d} = \mbox{Gamma}(1,1)$, letting $\epsilon_1, \ldots,\epsilon_{C}\stackrel{iid}\sim \mbox{Exp}(1)$ is the same (e.g., 
as proved in Lemma IV.3 of \citet{zhou2012negative}) in distribution as
 letting $$ \epsilon_i = \pi_i \epsilon,~~~\text{for $i=1,\ldots,{C},$ ~~~where } \piv\sim\mbox{Dirichlet }(\mathbf{1}_{{C}}),~\epsilon\sim \mbox{Gamma}({C},1),$$
and $\argmin\nolimits_{i}\pi_i e^{-\phi_i}=\argmin\nolimits_{i} \epsilon \pi_i e^{-\phi_i}$. Thus using Rao-Blackwellization \citep{casella1996rao}, we can re-express the gradient in 
\eqref{eq:Grad_1} as
\ba{
\nabla_{\phi_{c}} \mathcal{E}(\phiv)&= \E_{ \epsilon\sim \text{Gamma}({C},1), ~\piv\sim\text{Dirichlet}(\mathbf{1}_{C})}[f{}(\argmin\nolimits_{i} \epsilon \pi_i e^{-\phi_i} ) (1-\epsilon \pi_{c})]\notag\\
&= \E_{ \piv\sim\text{Dirichlet}(\mathbf{1}_{C})}[f{}(\argmin\nolimits_{i} \pi_i e^{-\phi_i} )(1-{C} \pi_{c})]. \notag\\
& = \E_{\piv\sim\text{Dirichlet}(\mathbf{1}_{C})}[f{}(\argmin\nolimits_{i} \pi^{_{{c} \leftrightharpoons j}}_i e^{-\phi_i} ) (1-{C} \pi_j)],
\label{eq:Grad_2}
}
where $j\in\{1,\ldots,{C}\}$ is an arbitrarily selected reference category, whose selection does not depends on $\piv$ and $\phiv$.

Another useful observation of the paper is that the function $$b(\piv,\phiv,j) = \frac{1}{{C}}\sum_{m=1}^{C} f{}(\argmin\nolimits_{i} \pi^{_{m \leftrightharpoons j}}_i e^{-\phi_i} ) (1-{C} \pi_j)$$ has zero expectation, as
\ba{
\E_{\piv\sim\text{Dirichlet}(\mathbf{1}_{C})}[b(\piv,\phiv,j)] &= \E_{\piv\sim\text{Dirichlet}(\mathbf{1}_{C})}\left[f{}(\argmin\nolimits_{i} \pi_i e^{-\phi_i} ) \sum_{m=1}^{C}\left(\frac{1}{{C}}- \pi_m\right)\right]
 = 0.
}
Using $\E[b(\piv,\phiv,j)]$ as the baseline function and subtracting it from \eqref{eq:Grad_2} leads to \eqref{eq:ARM-cat}. 
We now conclude the proof of Theorem~\ref{th:AR} for the AR estimator, and Equation \ref{eq:ARM-cat} for the ARS estimator. Once the ARS estimator is proved, Theorem \ref{theorem 3} for the ARSM estimator directly follows. 

\begin{proof}[Proof of Corollary \ref{cor:ARM}]
Note that letting $(u,1-u)\sim \mbox{Dir}(1,1) $ is the same as letting $u\sim \mbox{Uniform}(0,1)$.
Thus regardless of whether we choose Category 1 or Category 2 for as the reference category, we have 
\ba{
\nabla_{\phiv_1} \mathcal{E}(\phiv) = \E_{u\sim\text{Uniform}(0,1)}[f(\argmin(u,\sigma(\phi_1-\phi_2))-f(\argmin(1-u,\sigma(\phi_1-\phi_2))](1/2-u)
}
and $\nabla_{\phiv_2} \mathcal{E}(\phiv)=-\nabla_{\phiv_1} \mathcal{E}(\phiv)$.
Denote $\phi=\phi_1-\phi_2$ and $\eta=\phi_1+\phi_2$, we have
$$
\nabla_{\phi} \mathcal{E}(\phiv) = \nabla_{\phi_1} \mathcal{E}(\phiv)\frac{ \partial{\phi_1}}{\partial{\phi}} +\nabla_{\phi_2} \mathcal{E}(\phiv)\frac{ \partial{\phi_2}}{\partial{\phi}} =\nabla_{\phiv_1} \mathcal{E}(\phiv).
$$

\end{proof}

\section{Fast Computation for the Swap Step} 
\label{sec:fast compute}
Computing the pseudo actions $z^{_{{c} \leftrightharpoons j}} =\argmin\nolimits_{i} \pi_{i}^{_{{c} \leftrightharpoons j}}e^{-\phi_{i}}$ due to the swap operations can be efficiently realized: we first compute $o_{ij}=\ln \pi_{i} - \phi_{j}$, $z=\argmin_{i} (\ln \pi_{i} - \phi_i)$, and $o_{\min}=\ln \pi_{z} - \phi_z$; then for $m=1\ldots,C,~j<m$, compute
\ba{z^{_{{m} \leftrightharpoons j}}=
\begin{cases}
m, \text{ if } z\notin\{m,j\},~ \min\{o_{mj},o_{jm}\} < o_{\min},~o_{mj}\le o_{jm};\\
j, \text{ if } z\notin\{m,j\},~ \min\{o_{mj},o_{jm}\} < o_{\min},~o_{mj}>o_{jm};\\
\argmin\nolimits_{i} (\ln\pi_{i}^{_{m\leftrightharpoons j}} {-\phi_{i}}), \text{ if } z\in\{m,j\} 
;\\
z, \text{ otherwise};\\
\end{cases}\notag
}
and let $z^{_{{j} \leftrightharpoons j}} =z$ for all $j$, and $z^{_{{m} \leftrightharpoons j}} = z^{_{{j} \leftrightharpoons m}} $ for all $j>m$.

\section{ARSM for Multivariate, Hierarchical, and Sequential Categorical Variables}
\label{sec:derivation1}
\subsection{ARSM for Multivariate Categorical Variables}

\begin{proposition}[AR, ARS, and ARSM for multivariate categorical\label{theorem 1b}] 
Denote $\zv=(z_1,\ldots, z_K)$, where $z_k\in\{1,\ldots,{C}\}$, as a $K$ dimensional vector of $C$-way categorical variables. Denote $\Pimat = (\piv_1,\ldots,\piv_K) \in\mathbb{R}^{C\times {K}}$ as a matrix obtained by concatenating $K$ column vectors $\piv_k=(\pi_{k1},\ldots,\pi_{k{C}})'$, 
and $\Phimat = (\phiv_1,\ldots,\phiv_K)\in\mathbb{R}^{C\times {K}}$ by concatenating 
 $\phiv_k=(\phi_{k1},\ldots,\phi_{k{C}})'$. With the multivariate AR estimator, 
 the gradient of 
\ba{\mathcal{E}(\Phimat)=\E_{\zv\sim\prod_{k=1}^K\emph{\text{Cat}}(z_k;\sigma(\phiv_{k}))}[f{}(\zv)] \label{eq:PHI}
} with respect to 
$\phi_{k{c}} $ 
 is expressed as 
 \ba{
 \nabla_{\phi_{k{c}}}\mathcal{E}(\Phimat)& = \E_{ \Pimat\sim
 \prod_{{k}=1}^{{K}}\emph{\text{Dir}}(\piv_k;\mathbf{1}_{C})}[ 
 f(\zv)(1-{C} \pi_{kc})],\notag\\
 z_k :&\textstyle= \argmin_{i\in\{1,\ldots,C\}} \pi_{ki} e^{-\phi_{ki}}. 
 }
 Denoting $\jv=(j_1,\ldots,j_K)$, where $j_k\in\{1,\ldots,C\}$ is a randomly selected reference category for dimension $k$, 
 the multivariate ARS estimator is expressed as
 \begin{equation}
\!\begin{aligned}
\nabla_{\phi_{k{c}}}\mathcal{E}(\Phimat)& = 
 \E_{ \Pimat\sim
 \prod_{{k}=1}^{{K}}\emph{\text{Dir}}(\piv_k;\mathbf{1}_{C})}[ 
 f_{\Delta }^{_{{c} \leftrightharpoons \jv}}( \Pimat) (1-{C} \pi_{kj_k})], 
\\
\textstyle f_{\Delta }^{_{{c} \leftrightharpoons \jv}}( \Pimat) :
 &\textstyle = f({\zv}^{_{{c} \leftrightharpoons \jv}} ) - \frac{1}{C}\sum_{m=1}^{C} f({\zv}^{_{{m} \leftrightharpoons \jv}} ),\\
 {\zv}^{_{{c} \leftrightharpoons \jv}} :& =( z^{_{{c} \leftrightharpoons j_1}} _{1},z^{_{{c} \leftrightharpoons j_2}} _2,\ldots, z^{_{{c} \leftrightharpoons j_K}} _{K}), \\
 z^{_{{c} \leftrightharpoons j_k}} _{k}:&\textstyle=\argmin_{i\in\{1,\ldots,C\}} \pi^{_{c \leftrightharpoons j_k}}_{ki} e^{-\phi_{ki}}. \end{aligned} \!\!\!\!\!\!\label{eq:ARM-catvector}
\end{equation}
Setting $\jv=j\mathbf{1}_K$ and averaging over all $j\in\{1,\ldots,C\}$, the multivariate ARSM estimator is expressed as
\ba{
\textstyle \nabla_{\phi_{k{c}}}\mathcal{E}(\Phimat) &= 
 \E_{ \Pimat\sim
 \prod_{{k}=1}^{{K}}{\emph{\text{Dir}}}(\piv_k;\mathbf{1}_{C})}
 \textstyle 
 \big[ \sum_{j=1}^C 
 f_{\Delta }^{_{{c} \leftrightharpoons (j\mathbf{1}_K)}} ( \Pimat)(\frac{1}{C}-\pi_{kj})\big]. 
 \label{eq:ARM-catvector_1}
}

 \end{proposition}

Note to obtain $\nabla_{\phi_{k{c}}}\mathcal{E}(\Phimat)$ for all $k$ and $c$ based on the ARS estimator in \eqref{eq:ARM-catvector}, we only need to evaluate $f({\zv}^{_{{1} \leftrightharpoons \jv}} ),\ldots, f({\zv}^{_{{C} \leftrightharpoons \jv}} )$. Thus regardless of how large $K$ is, to obtain a single Monte Carlo sample estimate of the true gradient, one needs to evaluate the reward function $f(\cdotv)$ as few as zero time, which happens when the number of unique vectors in $\{{\zv}^{_{{c} \leftrightharpoons \jv}} \}_{c=1,C}$ is one, and as many as $C$ times, which happens when all ${\zv}^{_{{c} \leftrightharpoons \jv}} $ are different from each other. Similarly, if the ARSM estimator in \eqref{eq:ARM-catvector_1} is used, the number of times one needs to evaluate $f(\cdotv)$ is between zero and $C(C-1)/2+1$. In the multivariate setting where $\zv\in\{1,\ldots,C\}^K$, we often choose a relatively small $C$, such as $C=10$, but allows $K$ to be as large as necessary, such as $K=100$. Thus even $C^K$, the number of unique $\zv$'s, could be enormous when $K$ is large, both the ARS and ARSM estimators remain computationally efficient; this differs them from estimators, such as the one in \citet{titsias2015local}, that are not scalable in the dimension $K$.

\subsection{ARSM for Categorical Stochastic Networks}
Let us construct a $T$-categorical-stochastic-layer network as \ba{
 &\textstyle q_{ \Phimat_{1:T}}(\zv_{1:T}\given \xv) = 
\prod_{t=1}^{T}q(\zv_{t} \given \Phimat_{t}),~\Phimat_t : = \mathcal{T}_{\wv_t}(\zv_{1:t-1}), \notag\\
&~~~~~~~~~~\small \textstyle q(\zv_{t} \given \Phimat_t) :=\prod_{k=1}^{K_t} \mbox{Cat}(z_{tk}; \sigma(\phiv_{tk})), \label{eq:ARM_T}
}
where $\zv_0:=\xv$, $\zv_t:=(z_{t1},\ldots,z_{tK_t})'\in\{1,\ldots,C\}^{K_t}$ is a $K_t$-dimensional $C$-way categorical vector at layer $t$, $\phiv_{tk}:=(\phi_{tk1},\ldots,\phi_{tk{C}})' \in \mathbb{R}^{C}$ is the 
parameter vector for dimension $k$ at layer $t$, $\Phimat_t := \big(\phiv_{t1},\ldots,\phiv_{tK_t}\big)\in\mathbb{R}^{C\times K_t}$, and $\mathcal{T}_{\wv_t}(\cdotv)$ represents a function parameterized by $\wv_t$ that deterministically transforms $\zv_{t-1}$ to $\Phimat_t$. In this paper, we will define $\mathcal{T}_{\wv_t}(\cdotv)$ with a neural network. 
 
\begin{proposition}
For the categorical stochastic network defined in \eqref{eq:ARM_T}, the ARSM gradient of the objective
\baa{
&\mathcal{E}( \Phimat_{1:T}) 
 = \E_{\zv_{1:T}\sim q_{ \Phimat_{1:T} }(\zv_{1:T}\given \xv)}\left[f(\zv_{1:T})\right]
}
with respect to $\wv_t$ 
can be expressed as $ \nabla_{\wv_t}\mathcal{E}(\Phimat_{1:T}) = \nabla_{\wv_t} \big(\sum_{k=1}^{K_t}\sum_{c=1}^C(\nabla_{\phi_{tk{c}}}\mathcal{E}(\Phimat_{1:T}) ) \phi_{tkc}\big)$, where
\ba{
\textstyle \nabla_{\phi_{tk{c}}}\mathcal{E}(\Phimat_{1:T}) &
 = \E_{ \Pimat_t\sim
 \prod_{{k}=1}^{{K_t}}{\emph{\text{Dir}}}(\piv_{tk};\mathbf{1}_{C})}
 \textstyle 
 \big[ \sum_{j=1}^C 
 f_{t\Delta}^{_{{c} \leftrightharpoons j}}( \Pimat_t)\big(\frac{1}{C}-\pi_{tkj}\big)\big],
}
where $\piv_{tk}=(\pi_{tk1},\ldots,\pi_{tkC})'$ is the Dirichlet distributed probability vector for dimension $k$ at layer $t$ and
\bas{
\textstyle f_{t\Delta}^{_{{c} \leftrightharpoons j}}( \Pimat_t):
 &\textstyle = f(Z_t^{_{{c} \leftrightharpoons j}} ) - \frac{1}{C}\sum_{m=1}^{C} f(Z_t^{_{{m} \leftrightharpoons j}} ),\\
Z_t^{_{{c} \leftrightharpoons j}}: & =\{{\zv}_{1:t-1}, \zv_{t:T}^{_{{c} \leftrightharpoons j}} \},~~
 {\zv}_{1:t-1} \sim q_{\Phimat_{1:t-1}}(\zv_{1:t-1} \given \xv) ,\\
 {\zv}_{t}^{_{{c} \leftrightharpoons j}}:& =( z_{t1}^{_{{c} \leftrightharpoons j}},\ldots, z_{tK_t}^{_{{c} \leftrightharpoons j}})', \\
z_{tk}^{_{{c} \leftrightharpoons j}}:&\textstyle=\argmin_{i\in\{1,\ldots,C\}} \pi_{tki}^{_{c \leftrightharpoons j}}e^{-\phi_{tki}}, \\
 \zv_{t+1:T}^{_{{c} \leftrightharpoons j}} &\sim q_{\Phimat_{t+1:T}}(\zv_{t+1:T} \given {\zv}_{1:t-1}, \zv_{t}^{_{{c} \leftrightharpoons j}}).
 }
 \label{col: TKC}
\end{proposition}\vspace{-6mm}

\subsection{Proofs}
\label{sec:ARMvector}
Below we show how to generalize Theorem \ref{theorem 3} for a univariate categorical variable to Proposition \ref{theorem 1b} for multivariate categorical variables, and Proposition \ref{col: TKC} for hierarchical multivariate categorical variables. 

\begin{proof}[Proof of Proposition \ref{theorem 1b}]
For the expectation in \eqref{eq:PHI}, since $z_{k}$ are conditionally independent given $\phiv_{k}$, we have
\ba{
\nabla_{ \phi_{kc}}\mathcal{E}(\Phimat)&=\E_{\zv_{\backslash k}\sim\prod_{k'\neq k}{\text{Discrete}}(z_{k'};\sigma(\phiv_{k'}))} \big[ \nabla_{ \phi_{kc}} \E_{z_{k}\sim {\text{Cat}}(\sigma(\phiv_{k}))} [f{}(\zv)]\big ]. \label{eq:ARM_grad_M{C}}
}
Using Theorem \ref{theorem 3} to compute the gradient in the above equation directly leads to
\ba{
\small
&\small\nabla_{\phi_{kc}}\mathcal{E}(\Phimat)=\E_{\zv_{\backslash k}\sim\prod_{k'\neq k}{\text{Discrete}}(z_{k'};\sigma(\phiv_{k'}))} \Big\{ \E_{ \piv_{k}\sim{\text{Dirichlet}}(\mathbf{1}_{C})}\Big[(f(\zv_{\backslash k},\zv_k^{c \leftrightharpoons j} )-\frac{1}{C} \sum_{m=1}^C f(\zv_{\backslash k},\zv_k^{m \leftrightharpoons j} ))(1-C \pi_{kj})\Big]\Big\},\label{eq:PHI_mv_0}
} 
The term inside $[\cdotv]$ of \eqref{eq:PHI_mv_0} can already be used to estimate the gradient, however, in the worst case scenario that all the elements of $\{\zv_k^{c \leftrightharpoons j}\}_{j=1,C}$ are different, it needs to evaluate the function $f(\zv_{\backslash k},\zv_k^{c \leftrightharpoons j})$ for $j=1,\ldots,C$, and hence $C$ times for each ${k}$ 
and $KC$ times in total. To reduce computation and simplify implementation, exchanging the order of the two expectations in \eqref{eq:PHI_mv_0}, we have 
\ba{ \textstyle
&\nabla_{ \phi_{kc}}\mathcal{E}(\Phimat)=\E_{ \piv_{k}\sim{\text{Dirichlet}}(\mathbf{1}_{C})}\left\{(1-C \pi_{kj})
\E_{\zv_{\backslash k}\sim\prod_{k'\neq k}{\text{Discrete}}(z_{k'};\sigma(\phiv_{k'}))}\left[f(\zv_{\backslash k},\zv_k^{c \leftrightharpoons j})-\frac{1}{C} \sum_{m=1}^C f(\zv_{\backslash k}, \zv_k^{m \leftrightharpoons j} )\right]\right\} \label{eq:PHI_mv}
}
\vspace{-2mm}
Note that 
\bas{
&~~~\E_{\zv_{\backslash k}\sim\prod_{k'\neq k}{\text{Discrete}}(z_{k'};\sigma(\phiv_{k'}))} [f(\zv_{\backslash k}, \zv_k^{c \leftrightharpoons j})] \\
&\textstyle=\E_{\epsilonv_{\backslash k}\sim\prod_{k'\neq k}\prod_{i=1}^{C}{\text{Exp}}(\epsilon_{k'i};e^{\phi_{k'i}}) }
\big[f\big( (z_{k'}=\argmin_{i\in\{1,\ldots,C\}} \epsilon_{k'i}e^{-\phi_{k'i}})_{k'\neq k} ,~\zv_k^{c \leftrightharpoons j} \big)\big]\\
&\textstyle=\E_{\epsilonv_{\backslash k}\sim\prod_{k'\neq k}\prod_{i=1}^{C}{\text{Exp}}(\epsilon_{k'i };e^{\phi_{k'i}}) }
\big[f\big( (z_{k'}=\argmin_{i\in\{1,\ldots,C\}} \epsilon_{k'i}^{_{(c \leftrightharpoons j)} }e^{-\phi_{k'i}})_{k'\neq k} ,~\zv_k^{c \leftrightharpoons j}\big)\big]\\
&\textstyle=\E_{\Pimat_{\backslash k}\sim\prod_{k'\neq k}{\text{Dirichlet}}(\piv_{k'};\mathbf{1}_{C})} 
\big[f\big( (z_{k'}=\argmin_{i\in\{1,\ldots,C\}} \pi_{k'i}^{_{(c \leftrightharpoons j)} }e^{-\phi_{k'i}})_{k'\neq k} ,~\zv_k^{c \leftrightharpoons j}\big)\big]\\
&\textstyle=\E_{\Pimat_{\backslash k}\sim\prod_{k'\neq k}{\text{Dirichlet}}(\piv_{k'};\mathbf{1}_{C})} 
\big[f\big(\zv_1^{c \leftrightharpoons j},\ldots,\zv_K^{c \leftrightharpoons j}\big)\big]
}
Plugging the above equation into \eqref{eq:PHI_mv} leads to a simplified representation as \eqref{eq:ARM-catvector_1} shown in Proposition \ref{theorem 1b}, with which, regardless of the dimensions ${C}$, we draw $\Pimat=\{\piv_1,\ldots,\piv_{K}\}$ once to produce correlated $\zv^{c \leftrightharpoons j}$'s, 
and evaluate the function $f(\cdotv)$ at most $ C$ times.
\end{proof}

\begin{proof}[Proof of Proposition \ref{col: TKC}]
For multi-layer stochastic network
$
q_{ \Phimat_{1:T}}(\zv_{1:T}\given \xv) = q_{ \Phimat_{1}}(\zv_1\given \xv)\Big[\prod\nolimits_{t=1}^{T-1}q_{\Phimat_{t+1}}(\zv_{t+1} \given \zv_{t})\Big],
$
the gradient of the $t$-th layer parameter $ \Phimat_{t}$ is 
\bas{
\nabla_{\Phimat_{t}} \mathcal{E}( \Phimat_{1:T}) = \E_{\zv_{1:t-1} \sim q(\zv_{1:t-1} | \xv)} \nabla_{\Phimat_{t}} \E_{q(\zv_t | \zv_{t-1})} f_t(\zv_{1:t})
}
where $f_t(\zv_{1:t}) = \bE_{q(\zv_{t+1:T}|\zv_t)}[f(\zv_{1:T})]$. To compute the ARSM gradient estimator, first draw a single sample $\zv_{1:t-1} \sim q(\zv_{1:t-1} \given \xv)$ if $t>1$ and compute the pseudo action vector for the $t$-th layer according to Proposition \ref{theorem 1b} as
\bas{
z_{tk}^{c \leftrightharpoons j}:&\textstyle=\argmin_{i\in\{1,\ldots,C\}} \pi^{_{c \leftrightharpoons j}}_{tki} e^{-\phi_{tki}}
}
for $c, j \in \{1, \ldots, C\}$. For each pseudo action vector $\zv_{t}^{c \leftrightharpoons j}$, sample $ \zv_{t+1:T}^{c \leftrightharpoons j} \sim q(\zv_{t+1:T} \given \zv_{t}^{c \leftrightharpoons j})$ and compute $f_t(\zv^{c \leftrightharpoons j}) = f(\zv_{1:t-1}, \zv_{t:T}^{c \leftrightharpoons j})$. Replacing $f(\zv^{c \leftrightharpoons j})$ in Proposition \ref{theorem 1b} with the $f_t(\zv^{c \leftrightharpoons j})$ leads to the gradient estimator in Proposition~\ref{col: TKC}.
\end{proof}

\begin{proof}[Proof of Proposition \ref{lem1}]
We first write the objective function $J(\thetav)$ in terms of the intermediate parameters $\phiv_t = \mathcal{T}_{\thetav}(\sv_t)$, and then apply the chain rule to obtain the policy gradient $\nabla_{\thetav}J(\thetav)$.
 Since
$$
J(\phiv_{0:\infty}) =\mathbb{E}_{\mathcal{P}(\sv_{0})\prod_{t=0}^{\infty}\mathcal{P}(\sv_{t+1}\given \sv_{t},a_{t})\text{Cat}(a_{t};\sigma(\phiv_{t})) }\left[\sum_{t=0}^\infty\gamma^t r(\sv_t,a_t)\right] 
$$
we have
\vspace{-2mm}
\ba{
\textstyle 
J(\phiv_{0:\infty})&=\mathbb{E}_{\mathcal{P}(\sv_{0})\left[\prod_{t'=0}^{t-1}\mathcal{P}(\sv_{t'+1}\given \sv_{t'},a_{t'}) \text{Cat}(a_{t'};\sigma(\phiv_{t'}))\right]}\left\{\E_{a_{t}\sim \text{Cat}(\sigma(\phiv_{t})) }\left[\sum_{t'=0}^{t-1}\gamma^{t'}r(s_{t'},a_{t'})+ \gamma^{t}Q(\sv_t,a_t)\right]\right\}\notag\\
&=\mathbb{E}_{\mathcal{P}(\sv_{0})\left[\prod_{t'=0}^{t-1}\mathcal{P}(\sv_{t'+1}\given \sv_{t'},a_{t'}) \text{Cat}(a_{t'};\sigma(\phiv_{t'}))\right]}\left\{\E_{a_{t}\sim \text{Cat}(\sigma(\phiv_{t})) }\left[\sum_{t'=0}^{t-1}\gamma^{t'}r(s_{t'},a_{t'})\right]\right\}\notag\\
&~~~~+\mathbb{E}_{\mathcal{P}(\sv_{0})\left[\prod_{t'=0}^{t-1}\mathcal{P}(\sv_{t'+1}\given \sv_{t'},a_{t'}) \text{Cat}(a_{t'};\sigma(\phiv_{t'}))\right]}\left\{\E_{a_{t}\sim \text{Cat}(\sigma(\phiv_{t})) }\left[ \gamma^{t}Q(\sv_t,a_t)\right]\right\},\label{eq:decompose}
}
\vspace{-3mm}
where $Q(\sv_t,a_t)$ is the discounted action-value function defined as $$Q(\sv_t,a_t):= 
\E_{\prod_{t'=t}^{\infty} \text{Cat}(a_{t'+1};\sigma(\phiv_{t'+1})) \mathcal{P}(\sv_{t'+1}\given \sv_{t'},a_{t'}) }\left[\sum_{t'=t}^\infty \gamma^{t'-t} r(\sv_{t'},a_{t'})\right].$$ The first summation term in \eqref{eq:decompose}
 can be ignored for computing $\nabla_{\phiv_t}J(\phiv_{0:\infty}) $, and the second one can be re-expressed as
\ba{
\E_{\mathcal{P}(\sv_{t}\given \sv_0,\pi_{\thetav}) \mathcal{P}(\sv_{0}) }\left\{\E_{a_{t}\sim \text{Cat}(\sigma(\phiv_{t})) }\left[ \gamma^{t}Q(\sv_t,a_t)\right]\right\},\label{eq:decompose1}
\vspace{-5mm}
}
where
$
\mathcal{P}(\sv_{t}\given \sv_0,\pi_{\thetav})$ is the marginal form of the joint distribution $ \prod_{t'=0}^{t-1}\mathcal{P}(\sv_{t'+1}\given \sv_{t'},a_{t'}) \text{Cat}(a_{t'};\sigma(\phiv_{t'})).
$
Applying Theorem 2 to \eqref{eq:decompose1}, we have
\ba{
\nabla_{\phiv_{tc}}J(\phiv_{0:\infty})& =
\E_{\mathcal{P}(\sv_{t}\given \sv_0,\pi_{\thetav}) \mathcal{P}(\sv_{0}) }\left\{\gamma^{t} \nabla_{\phiv_{tc}} \E_{a_{t}\sim \text{Cat}(\sigma(\phiv_{t})) }\left[ Q(\sv_t,a_t)\right]\right\}\notag\\
& = \E_{\mathcal{P}(\sv_{t}\given \sv_0,\pi_{\thetav}) \mathcal{P}(\sv_{0}) }\left\{\gamma^{t} \E_{\varpiv_t \sim \text{Dir}(\mathbf{1}_C) }
\left[g_{tc}\right]\right\},
\label{eq:decompose2}
\vspace{-3mm}
}
\vspace{-3mm}
where 
\bas{ \textstyle
g_{tc}:&=\sum_{j=1}^C f_{t\Delta}^{c\leftrightharpoons j}(\varpiv_t)\left(\frac{1}{C}-\varpi_{tj}\right),\\
f_{t\Delta}^{c\leftrightharpoons j}(\varpiv_t): &=Q(s_t,a_t^{c\leftrightharpoons j})-\frac{1}{C}\sum_{m=1}^C Q(s_t,a_t^{m\leftrightharpoons j}),\\
a_t^{_{{c} \leftrightharpoons j}}: &\textstyle= \argmin_{i\in\{1,\ldots,C\}} \varpi_{ti}^{_{{c} \leftrightharpoons j}} e^{-\phi_{ti}}.
}
Applying the chain rule, we obtain the gradient as
\ba{ \textstyle
&\nabla_{\thetav}J(\thetav) = \sum_{t=0}^\infty\sum_{c=1}^C \frac{\partial J(\phiv_{0:\infty})}{\partial\phi_{tc}}\frac{\partial\phi_{tc}}{\partial\thetav}\notag\\
&=\sum_{t=0}^\infty\sum_{c=1}^C \mathbb{E}_{\mathcal{P}(\sv_{0}) \mathcal{P}(\sv_{t}\given \sv_0,\pi_{\thetav}) }\left\{ \gamma^t \E_{\varpiv_t\sim \text{Dir}(\mathbf{1}_C) }
\left[ g_{tc}\right] 
\nabla_{\thetav} \phi_{tc}\right\}\notag\\
&=\sum_{t=0}^\infty \mathbb{E}_{\mathcal{P}(\sv_{0}) \mathcal{P}(\sv_{t}\given \sv_0,\pi_{\thetav}) }\left\{ \gamma^t \E_{\varpiv_t\sim \text{Dir}(\mathbf{1}_C) }
\left[ \nabla_{\thetav} \sum_{c=1}^C g_{tc} \phi_{tc}\right] 
\right\}\notag\\
&= \mathbb{E}_{\sv_t\sim \rho_{\pi}(\sv) }\left\{ \E_{\varpiv_t\sim \text{Dir}(\mathbf{1}_C) }
 \left[ \nabla_{\thetav} \sum_{c=1}^C g_{tc} \phi_{tc}\right] 
\right\},
}
where $\rho_{\pi}(\sv):= \sum_{t=0}^\infty \gamma^t \mathcal{P}(\sv_t=\sv \given \sv_0,\pi_{\thetav})$ is the unnormalized discounted state visitation frequency. This concludes the proof of the ARSM policy gradient estimator. The proof of the ARS policy gradient estimator can be similarly derived, omitted here for brevity. 
\end{proof}
\section{Additional Figures and Tables}

 \begin{table}[H]
 \caption{The constructions of variational auto-encoders. The following symbols ``$\rightarrow$'', ``$]$'', $)$'', and ``$\rightsquigarrow$'' represent deterministic linear transform, 
 leaky rectified linear units (LeakyReLU) \citep{maas2013rectifier} nonlinear activation, softmax nonlinear activation, and discrete stochastic activation, respectively, in the encoder; their reversed versions are used in the decoder. 
}\label{tab:Network}
 \centering

{
 \begin{tabular}{ccc}
 \toprule 
 &One layer & Two layers\\
 \midrule
 Encoder &784$\rightarrow$512]$\rightarrow$256]$\rightarrow$200)$\rightsquigarrow$200 & 784$\rightarrow$512]$\rightarrow$256]$\rightarrow$200)$\rightsquigarrow$200 $\rightarrow$ 200) $\rightsquigarrow$200 \\
 Decoder & 784$\leftsquigarrow$(784$\leftarrow$[512$\leftarrow$[256$\leftarrow$200 & 784$\leftsquigarrow$(784$\leftarrow$[512$\leftarrow$[256$\leftarrow$200 $\leftsquigarrow$ (200 $\leftarrow$ 200\\ 
 \bottomrule
 \end{tabular}}\vspace{-1mm}
\end{table}

\begin{figure}[H]
 \centering
 \includegraphics[width=0.8\textwidth,height=7cm]{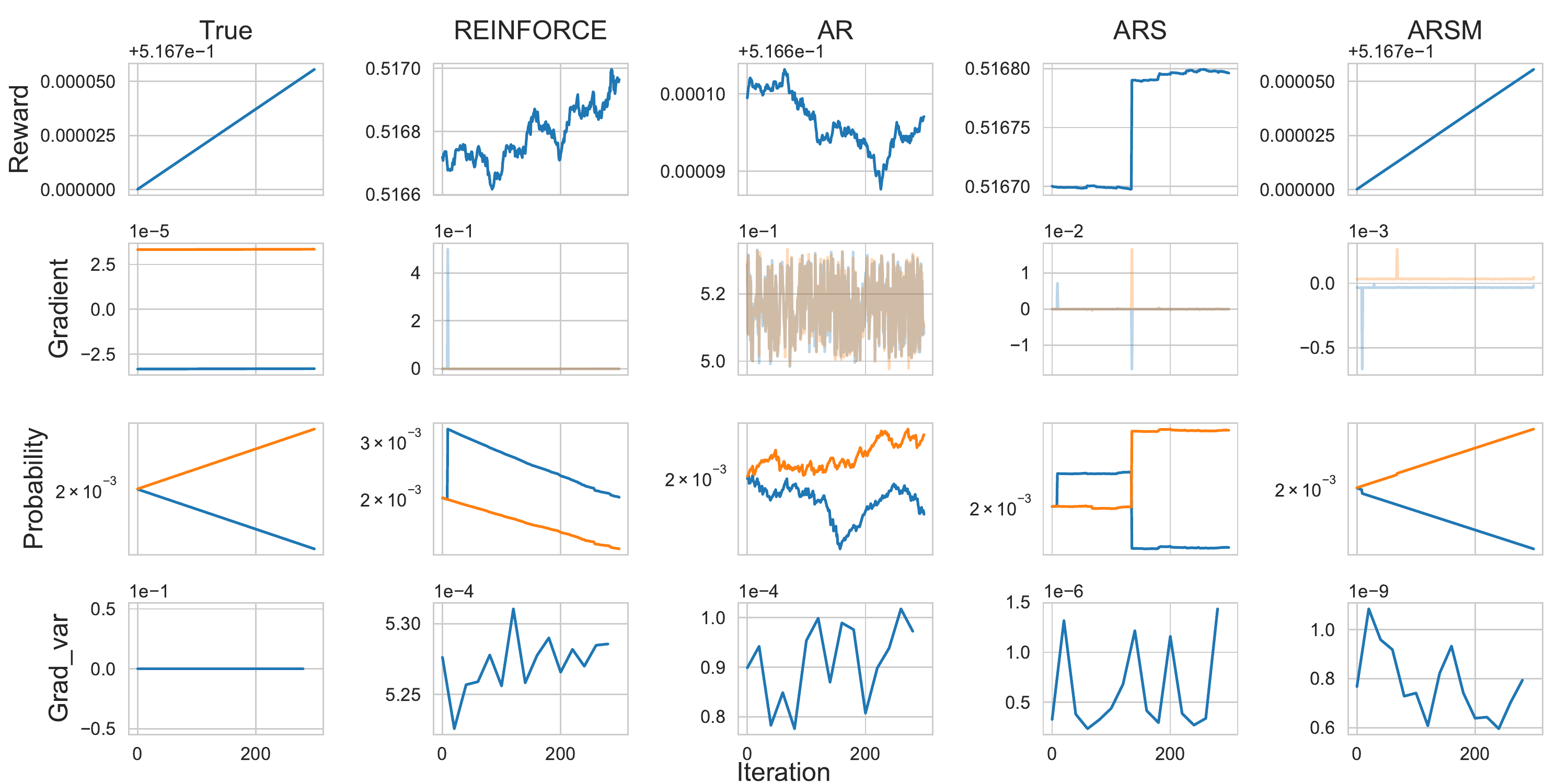}
 \caption{
Analogous plots to these in Figure \ref{fig:toy}, obtained with $C=1,000$.
 }
 \label{fig:toy1000}
\end{figure}

\begin{figure}[H]
 \centering
 \includegraphics[width=0.8\textwidth,height=7cm]{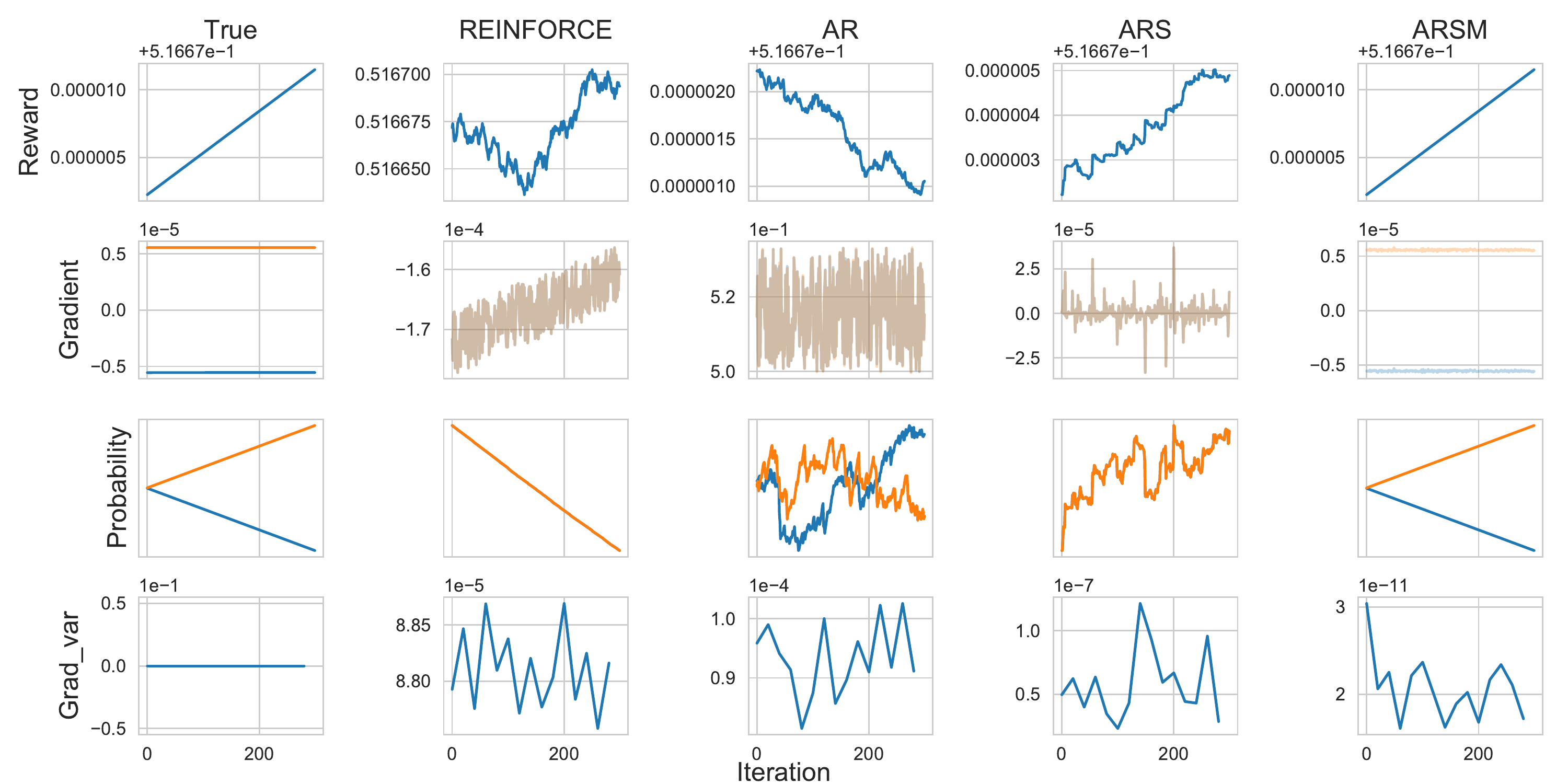}
 \caption{
 Analogous plots to these in Figure \ref{fig:toy}, obtained with $C = 10,000$.}
 \label{fig:toy10000}
\end{figure}

 \begin{figure}[H]
 \centering
 \includegraphics[width=0.45\textwidth,height=5cm]{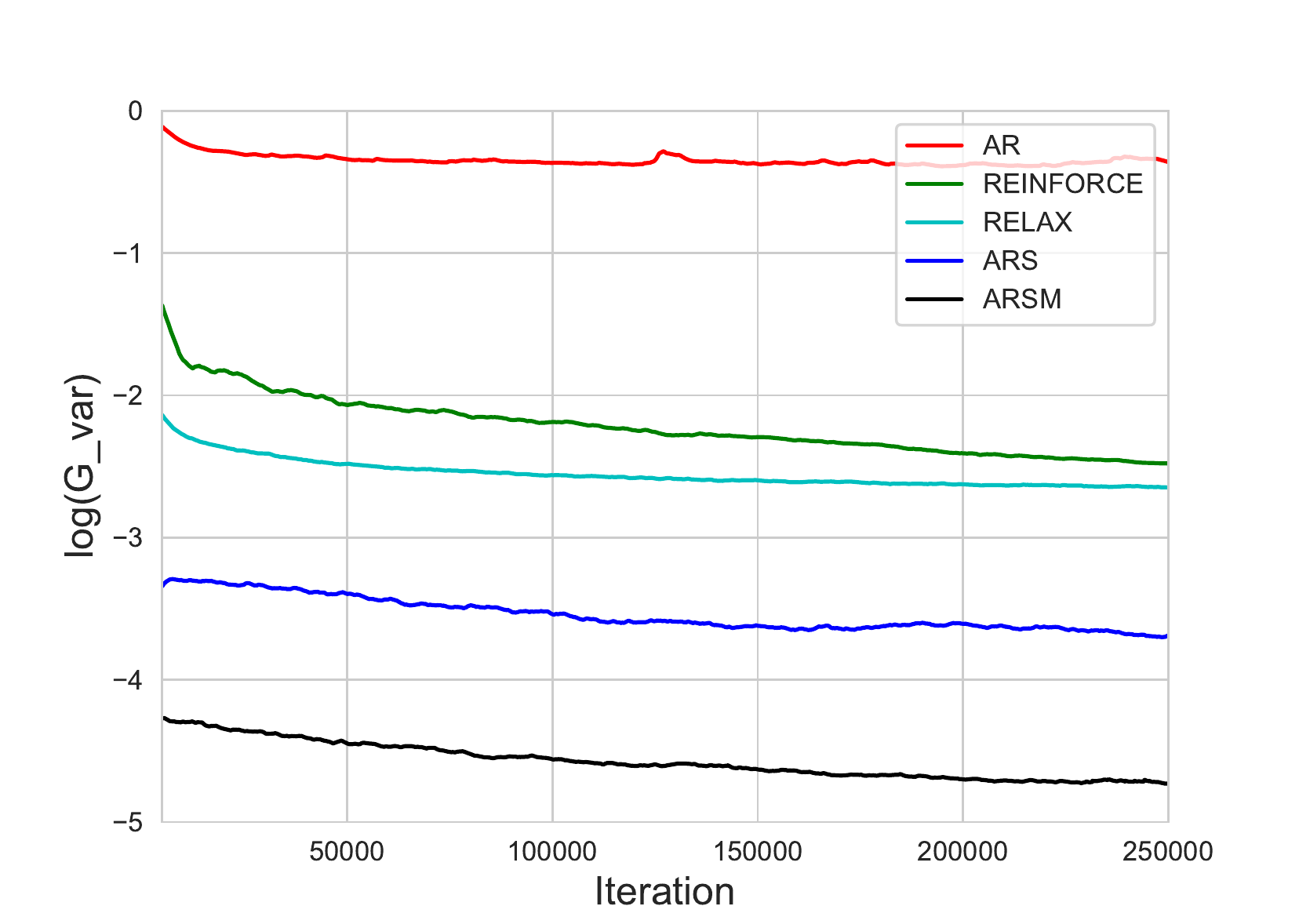}
 \vspace{-3mm}
 \caption{ Trace plots of the log variance of 
 various unbiased gradient estimators for categorical VAE on MNIST. The variance is estimated by exponential moving averages of the first and second moments with a decay factor of 0.999. The variance is averaged over all elements of the gradient vector.}
 \label{fig:vae_var}

\centering
\begin{tabular}{c}
\hspace{-2em}
\hspace{1em}\includegraphics[width=0.59\textwidth,height=4cm]{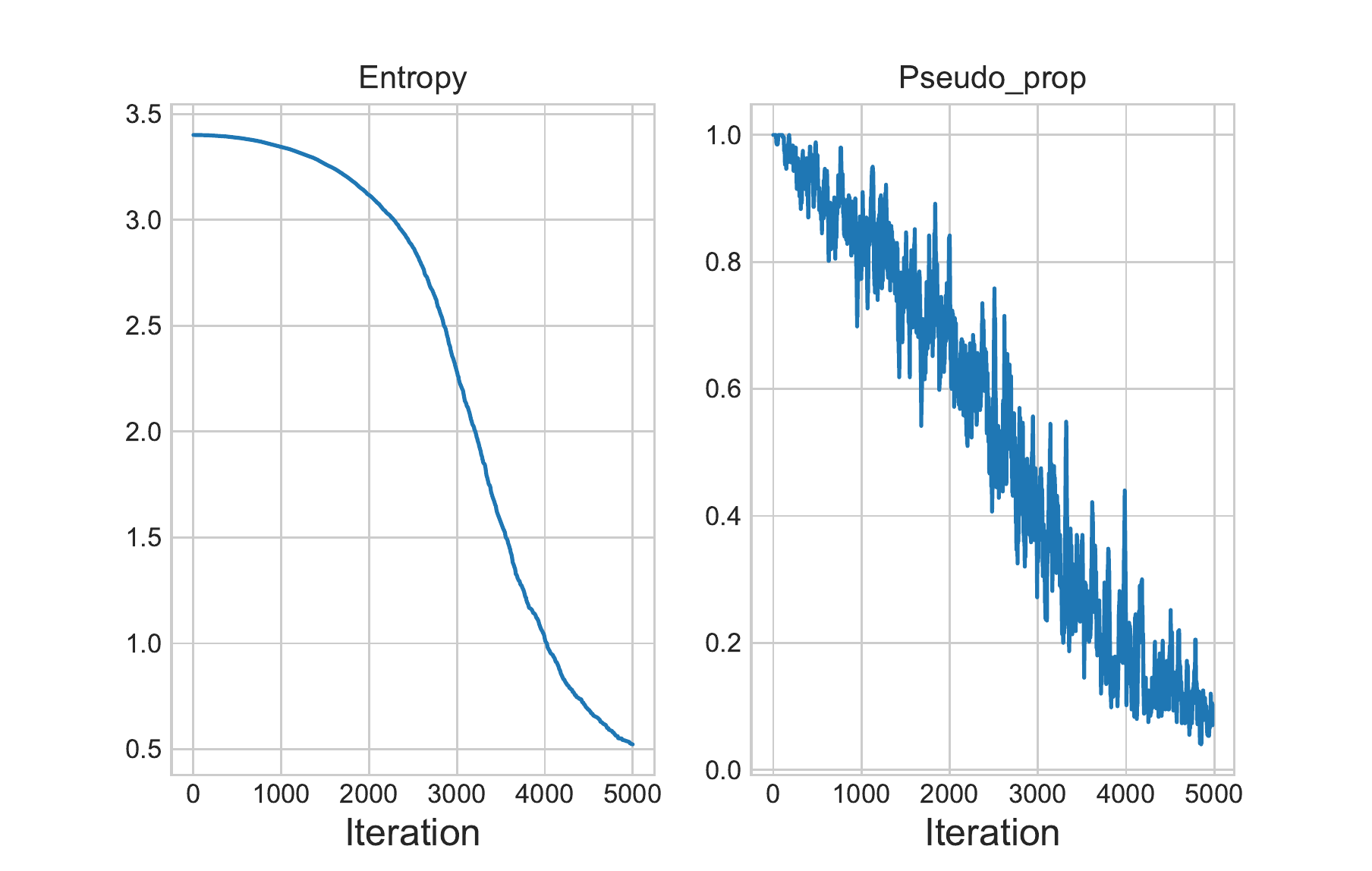}\hspace{-1em} \\ (a)\\
\includegraphics[width=0.6\textwidth,height=4cm]{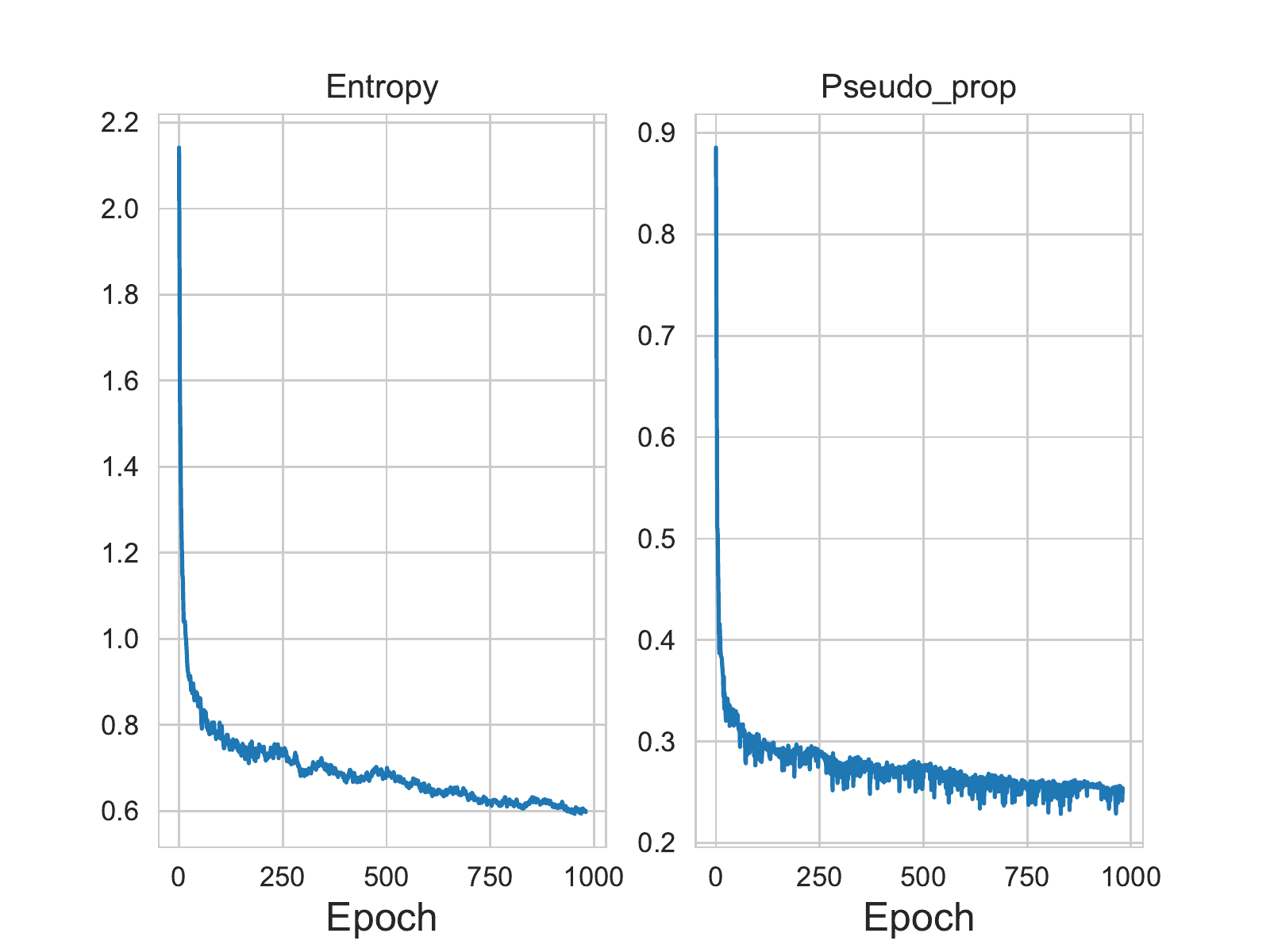}\hspace{-1em} \\(b)\\
\includegraphics[width=0.6\textwidth,height=4cm]{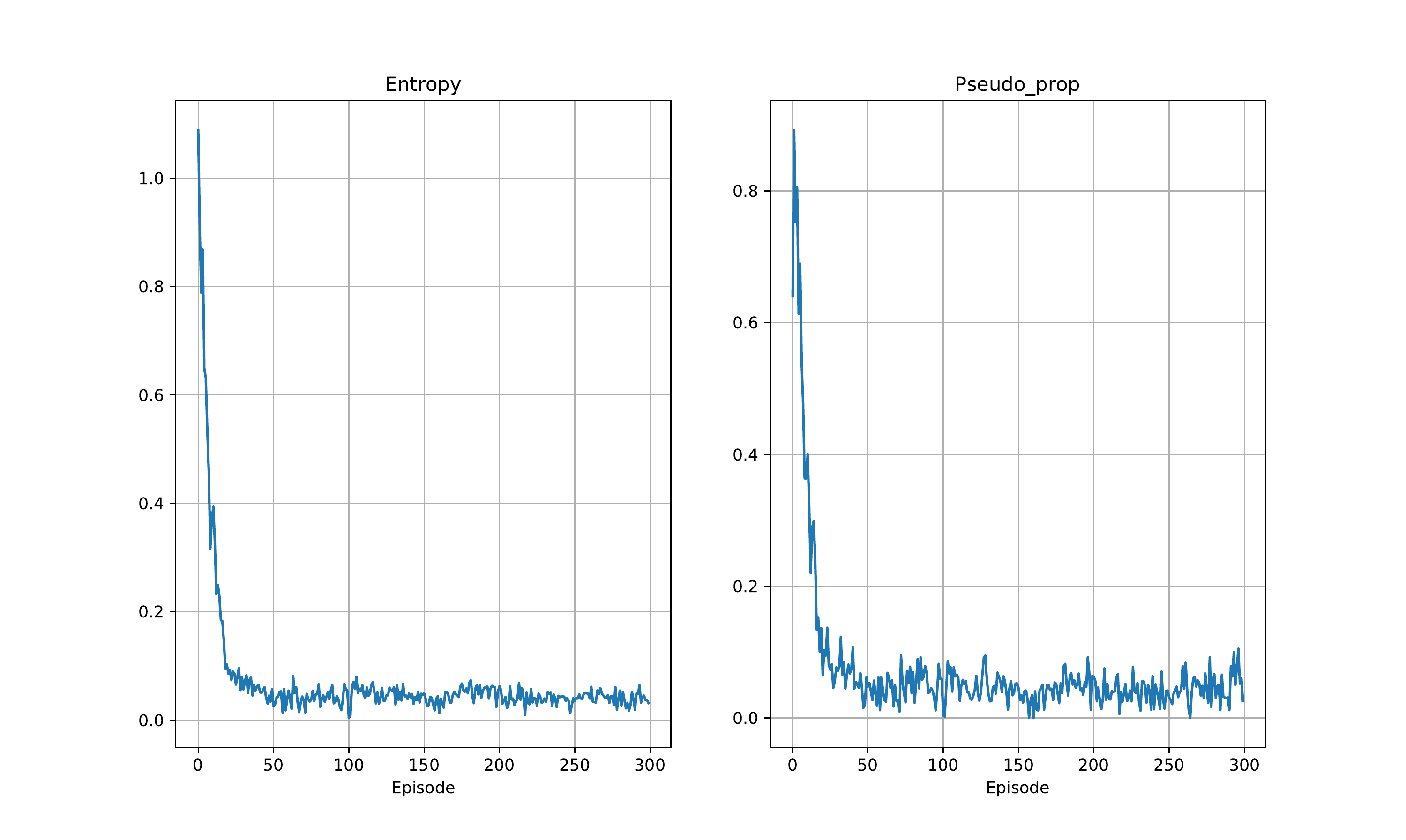}\hspace{-1em} \\(c)
\end{tabular} \vspace{-1mm}
\caption{The entropy of latent categorical distributions and the number of distinct pseudo actions, which differ from their corresponding true actions, both decrease as the training progresses. We plot the average entropy for $\{z_{tk}\}$ for all $t=1:T$ and $k=1:K$. The pseudo action proportion for the $k$-th categorical random variable at the $t$-th stochastic layer is calculated as the number of unique values in $\{z_{tk}^{c \leftrightharpoons j} \}_{c = 1:C, j=1:C}\backslash z_{tk}$ divided by $C-1$, the maximum number of distinct pseudo actions that differ from the true action $z_{tk}$. We plot the average pseudo action proportion for $\{z_{tk}\}$ for all $t=1:T$ and $k=1:K$. Subplots (a), (b), and (c) correspond to the Toy data ($T=K=1$, $C=30$), VAE with a single stochastic layer ($T=1$, $K=20$, $C=10$), and Acrobot RL task ($0\le T\le 500$, $K=1$, $C=3$); other settings yield similar trace plots.
}
\label{fig:entropy}
\end{figure}

 \begin{figure}[H]
 \centering
 \includegraphics[width=0.5\textwidth,height=4.5cm]{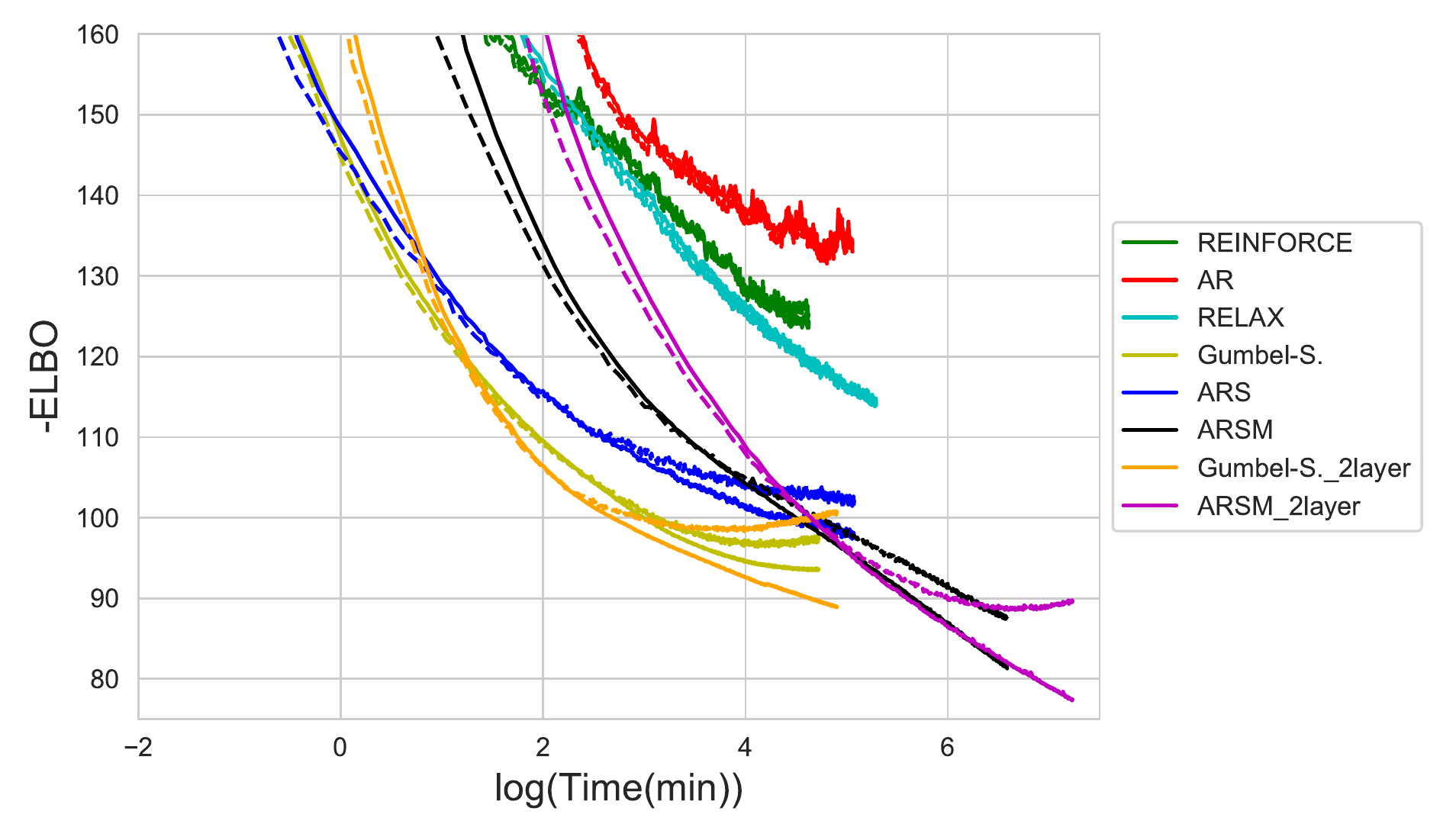}
 \caption{ Plots of $-$ELBOs (nats) on binarized MNIST against wall clock times on NVIDIA Tesla V100 GPU (analogous ones against training iterations are shown in Figure \ref{fig:vae}). The solid and dash lines correspond to the training and testing respectively (best viewed in color).}
 \label{fig:vae_time}
\end{figure}

\section{Algorithm}

{\begin{algorithm}[h] \small{
\SetKwData{Left}{left}\SetKwData{This}{this}\SetKwData{Up}{up}
\SetKwFunction{Union}{Union}\SetKwFunction{FindCompress}{FindCompress}
\SetKwInOut{Input}{input}\SetKwInOut{Output}{output}
\Input{Reward function $f(\zv;\thetav)$ parameterized by $\thetav$; }
\Output{Distribution parameter $\Phimat = (\phiv_1,\cdots,\phiv_K)\in\mathbb{R}^{C\times K}$ and reward function parameter $\thetav$ that maximize the expected reward as $\cE(\Phimat,\thetav) := \E_{\zv \sim \prod_{k=1}^K \text{Cat}(z_k;\sigma(\phiv_k))
}[f{}(\zv; \thetav)]$;}
\BlankLine
Initialize $\Phimat$ and $\thetav$ randomly;\;

\While{not converged}{
 Sample $\piv_k \sim \dir(\mathbf{1}_C)$ for $k = 1,\ldots,K$; 
 
 {Let $z_k = \argmin_{i\in\{1,\ldots,C\}}( \ln \pi_{ki}- \phi_{ki}) $ for $k = 1,\ldots,K$ to obtain the true action vector $\zv=(z_1,\ldots,z_k)$;}

 \If{Using the ARS estimator}{
 	
 	Using a single reference vector $\jv=(j_1,\ldots,j_K) $ for the variable-swapping operations, where all $j_k$ are uniformly at random selected from $\{1,\ldots,C\}$;
 	
	\For{$c=1,\ldots,C$
	(in parallel)}{
	
		Let $ z_k^{_{c \leftrightharpoons j_k}} = \argmin_{i\in\{1,\ldots,C\}} (\ln \pi^{_{c \leftrightharpoons j_k}}_{ki} -\phi_{ki})$ for $k = 1,\ldots,K$;
		
		Denote ${\zv}^{_{c \leftrightharpoons \jv}}=(z_1^{_{c \leftrightharpoons j_1}},\ldots, z_K^{_{c \leftrightharpoons j_K}}) $ as the $c$th pseudo action vector;

}
Let $\bar{f} = \frac{1}{C}\sum_{c=1}^C f({\zv}^{_{c \leftrightharpoons \jv}})$

Let 
$g_{\phi_{kc}} = \big(f({\zv}^{_{c \leftrightharpoons \jv}}) -\bar{f} \,\big) (1-C\pi_{kj_k})$ 
 for all $(k,c)\in\{(k,c)\}_{k=1:K,~c=1:C}$;

 }
 
 \If{Using the ARSM estimator}{
 Initialize the diagonal of reward matrix $F\in\mathbb{R}^{C\times C}$ with $f(\zv)$, which means letting $F_{cc}=f(\zv)$ for $c=1,\ldots, C$;

	\For{$(c,j)\in\{(c,j)\}_{{c}=1:{C},~j< c}$
	(in parallel)}{
			Let $\jv=j\mathbf{1}_K$, which means $j_k\equiv j$ for all $k\in\{1,\ldots,K\}$;
		
		Let $ z_k^{_{c \leftrightharpoons j}} = \argmin_{i\in\{1,\ldots,C\}} (\ln \pi^{_{c \leftrightharpoons j}}_{ki} -\phi_{ki})$ for $t = 1,\ldots,K$;
		
		Denote ${\zv}^{_{c \leftrightharpoons \jv}}=(z_1^{_{c \leftrightharpoons j}},\ldots, z_K^{_{c \leftrightharpoons j}}) $ as the $(c,j)$th pseudo action vector;

		Let $F_{cj} = F_{jc}= f({\zv}^{_{c \leftrightharpoons \jv}})$; 
}

Let $\bar F_{\cdotv j}= \frac{1}{C} \sum_{c=1}^C F_{cj}$ for $j=1,\ldots, C$;

Let $g_{\phi_{kc}} = \sum_{j=1}^C (F_{cj} - \bar F_{\cdotv j} ) (\frac{1}{C}-\pi_{kj})$ for all $(t,c)\in\{(t,c)\}_{k=1:K,~c=1:C}$; 

 }

$\Phimat = \Phimat + \rho_{\phi} \{g_{\phiv_{kc}}\}_{k=1:T,~c=1:C},~~~~$ with step-size $\rho_{\phi} $;

 $\thetav = \thetav + \eta_{\theta} \nabla_{\thetav}f{}(\zv;\thetav),~~~~$ with step-size $\eta_{\theta}$
}
*Note if the categorical distribution parameter $\Phimat$ itself is defined by neural networks with parameter $\wv$, standard backpropagation can be applied to compute the gradient with $\frac{\partial \mathcal{E}( \Phimat,\thetav)}{\partial \wv} = \frac{\partial \mathcal{E}( \Phimat,\thetav)}{\partial \Phimat}\frac{\partial \Phimat}{\partial \wv} \approx \nabla_{\wv}\big(\sum_{k=1}^K\sum_{c=1}^C g_{\phiv_{kc} } \phi_{kc}\big)$. 
\caption{ARS/ARSM gradient for $K$-dimensional $C$-way categorical vector $\zv = (z_1,\cdots,z_K)$, where $z_k\in\{1,\ldots,C\}$. }\label{alg:ARSM}}
\end{algorithm}
}

{\begin{algorithm}[h] \small{
\SetKwData{Left}{left}\SetKwData{This}{this}\SetKwData{Up}{up}
\SetKwFunction{Union}{Union}\SetKwFunction{FindCompress}{FindCompress}
\SetKwInOut{Input}{input}\SetKwInOut{Output}{output}
\Input{Reward function $f(\zv_{1:T};\thetav)$ parameterized by $\thetav$; 
}
\Output{Distribution parameter $\Phimat_{t} = (\phiv_{t1},\cdots,\phiv_{tK})'\in\mathbb{R}^{K\times C}$ and 
parameter $\thetav$ that maximize the expected reward as $\cE(\Phimat_{1:T},\thetav) := \E_{\zv \sim q_{ \Phimat_{1}}(\zv_1\given \xv)[\prod\nolimits_{t=1}^{T-1}q_{\Phimat_{t+1}}(\zv_{t+1} \given \zv_{t})]
)
}[f{}(\zv; \thetav)]$; $q_{\Phimat_{t}}(\zv_{t} \given \zv_{t-1}) = \prod_{k =1}^K \text{Categorical}(z_{tk} | \sigma(\phiv_{tk}(\zv_{t-1})) )
$;}
\BlankLine
Initialize $\Phimat_{1:T}$ and $\thetav$ randomly;\;

\While{not converged}{

 \For{t = 1 : T}{
Sample $\piv_{tk} \sim \dir(\mathbf{1}_C)$ for $k = 1,\ldots,K$; 
 {Let $z_{tk} = \argmin_{i\in\{1,\ldots,C\}}( \ln \pi_{tki} - \phi_{tki} ) $ for $k = 1,\ldots,K$ to obtain the true action vector $\zv_t =(z_{t1} ,\ldots,z_{tK} )$;}
 
 \If{Using the ARS estimator}{
 
 Let $\jv_t=(j_{t1},\ldots,j_{tK}) $, where $j_{tk}\in\{1,\ldots,C\}$ is a randomly selected reference category for dimension $k$ at layer $t$.

	\For{$c=1,\ldots,C$
		(in parallel)}{
	
		Let $ z_{tk}^{c \leftrightharpoons j_{tk}}:\textstyle=\argmin_{i\in\{1,\ldots,C\}} \pi^{_{c \leftrightharpoons j_{tk}}}_{tki} e^{-\phi_{tki}}$ for $k = 1,\ldots,K$;
		
		Denote $\zv_{t}^{c \leftrightharpoons \jv_t}=( z_{t1}^{c \leftrightharpoons j_{t1}},\ldots,z_{tK}^{c \leftrightharpoons j_{tK}}) $ as the $c$th pseudo action vector;

}
Let $\bar{f}_t = \frac{1}{C}\sum_{c=1}^C f(\zv_{t}^{c \leftrightharpoons \jv_{t}})$

Let 
$g_{\phi_{tkc}} = \big(f(\zv_t^{{c \leftrightharpoons \jv_t}}) -\bar{f}_t \,\big) (1-C\pi_{kj_{tk}})$ 
 for all $(k,c)\in\{(k,c)\}_{k=1:K,~c=1:C}$; 
 }	
{
\If{Using the ARSM estimator}{

 Let $F ^{(t)} \in\mathbb{R}^{C\times C}$ 
 
If $t>1$, sample $\zv_{1:t-1} \sim q(\zv_{1:t-1} | \xv)$ ;

 \For{$(c,j)\in\{(c,j)\}_{{c}=1:{C},~j \leq c}$
	(in parallel)}{
			Let $\jv=j\mathbf{1}_K$, which means $j_k\equiv j$ for all $k\in\{1,\ldots,K\}$;
		
		Let $ z_{tk}^{c \leftrightharpoons j}:\textstyle=\argmin_{i\in\{1,\ldots,C\}} \pi^{_{c \leftrightharpoons j}}_{tki} e^{-\phi_{tki}}$ for all $k\in\{1,\ldots,K\}$;
		
		Denote $\zv_{t}^{c \leftrightharpoons j}=( z_{t1}^{c \leftrightharpoons j},\ldots,z_{tK}^{c \leftrightharpoons j}) $ as the $(c,j)$th pseudo action vector;
		
		If $t<T$, sample $\zv_{t+1:T}^{c \leftrightharpoons j} \sim q(\zv_{t+1:T} | \zv_{t}^{c \leftrightharpoons j}) $;
		
				Let $F^{(t)} _{cj} = F^{(t)} _{jc}= f(\zv_{1:t-1}, \zv_{t:T}^{c \leftrightharpoons j} )$;

Let $\bar F^{(t)} _{\cdotv j}= \frac{1}{C} \sum_{c=1}^C F^{(t)} _{cj}$ for $j=1,\ldots, C$;

Let $g_{\phi_{tkc}} = \sum_{j=1}^C (F^{(t)} _{cj} - \bar F^{(t)} _{\cdotv j} ) (\frac{1}{C}-\pi_{kj})$ for all $(k,c)\in\{(k,c)\}_{k=1:K,~c=1:C}$; 
 }
 }
 $\Phimat _{t}= \Phimat _{t}+ \rho_{\Phimat _{t}} \{g_{\phi_{tkc}}\}_{k=1:K,~c=1:C},~~~~$ with step-size $\rho_{\Phimat _{t}} $;
}}

 $\thetav = \thetav + \eta_{\theta} \nabla_{\thetav}f{}(\zv;\thetav),~~~~$ with step-size $\eta_{\theta}$
}

\caption{ARS/ARSM gradient for $T$ layer $K$-dimensional $C$-way categorical vector $\zv_t = (z_{t1},\cdots,z_{tK})$, where $t \in \{1,\ldots, T\}$, $z_{tk}\in\{1,\ldots,C\}$. }\label{alg:ARSM1}}
\end{algorithm}
}

{\begin{algorithm}[h] {

\SetKwData{Left}{left}\SetKwData{This}{this}\SetKwData{Up}{up}
\SetKwFunction{Union}{Union}\SetKwFunction{FindCompress}{FindCompress}
\SetKwInOut{Input}{input}\SetKwInOut{Output}{output}
\Input{
Maximum number of 
state-pseudo-action rollouts 
$S_{\max}$ allowed in a single iteration; }
\Output{
Optimized policy parameter $\thetav$;
}
\BlankLine

\While{not converged}{
Given a random state $\sv_0$ and environment dynamics $\mathcal{P}(\sv_{t+1}\given a_t,\sv_{t})$, 
we run an episode till its termination (or a predefined number of steps) by sampling a true-action trajectory $(a_0,\sv_1,a_1,\sv_2,\ldots)$ given policy 
$\pi_{\thetav}(a_t\given \sv_t) := \mbox{Cat}(a_t; \sigma(\phiv_t)),~~\phiv_t:=\mathcal{T}_{\thetav}(\sv_t)$, where we sample each $a_t$ by 
 first 
sampling $(\varpi_{t1},\ldots,\varpi_{t{c}})\sim \mbox{Dir}(\mathbf{1}_{C})$ and then letting 
 $a_t = \argmin_{i\in\{1,\ldots,C\}} (\ln \varpi_{ti}-\phi_{ti})$; \;

 Record the termination time step of the episode as $T$, and set the rollout set as $H=[]$ and $S_0=0$;
 
 \For{ $t\in\text{RandomPermute}(0,\ldots,T)$ }{

 	Let $A_t = \{(c,j)\}_{{c}=1:{C},~j<c}$

 Initialize $a_{t}^{_{c \leftrightharpoons j}}=a_t$ for all ${c}$ and $j$; 

 \For{ $(c,j) \in A_t 
 $ (in parallel)}{
Let $ a_{t}^{_{c \leftrightharpoons j}}=a_t^{_{j \leftrightharpoons c}}=\argmin\nolimits_{i\in\{1,\ldots,C\}} (\ln \varpi_{ti}^{_{{c} \leftrightharpoons j}} {-\phi_{ti}})
$
}
Let $S_t=\mbox{unique}(\{ a_{t}^{_{c \leftrightharpoons j}}\}_{c,j})\backslash a_t$, which means $S_t$ is the set of all unique values in $\{a_{t}^{_{c \leftrightharpoons j}}\}_{c,j}$ that are different from the true action $a_t$;
Denote the cardinality of $S_t$ as $|S_t|$, where $0\le |S_t|\le {C}-1$ ;

\uIf{$S_0+|S_t|\le S_{\max}$}{
$S_0=S_0+|S_t|$

Append $t$ to $H$

}
\Else{
\textbf{break}
}
}
\For{ $t \in H$ (in parallel)}{

Initialize $R_{tmj}= \hat Q(\sv_t,a_t) = \sum_{t'=t}^T \gamma^{t'-t} r(\sv_{t'},a_{t'})$ for all $m,j\in\{1,\ldots,{C}\}$ \;
\For{$k\in\{1,\ldots,|S_t|\}$ (in parallel)}{
Let $\tilde a_{tk} = S_t(k)$ be the $k$th unique pseudo action at time $t$;

Evaluate $\hat Q (\sv_t, \tilde a_{tk})$, which in this paper is set as $r(\sv_{t},\tilde a_{tk})+ \gamma \sum_{t'=t+1}^\infty \gamma^{t'-(t+1)} r(\tilde \sv_{t'},\tilde a_{t'})$, where $(\sv_{t},\tilde a_{tk},\tilde \sv_{t+1},\tilde a_{t+1},\ldots)$ is a state-pseudo-action rollout generated by taking pseudo action $\tilde a_{tk}$ at state $\sv_t$ and then following the environment dynamics and policy $\pi_{\thetav}$; 

Let $
R_{tmj} = 
\hat Q (\sv_t, \tilde a_{tk}) $ for all $(m,j)$ in $\{(m,j): a_{t}^{_{{m} \leftrightharpoons j_t}} =\tilde a_{tk}\}$;
}

}

{
Esimate the ARSM policy gradient as 

 	$$
\nabla_{\thetav}J(\thetav)\approx \nabla_{\thetav} \left\{\sum_{t\in H} 
\sum_{{c}=1}^{C}\left[\sum_{j=1}^{C} \left(
R_{tcj} 
-\frac{1}{{C}}\sum_{m=1}^{C}
R_{tmj} 
\right)
 \left(\frac{1}{{C}}- \varpi_{tj}\right)
\right]\phi_{tc}\right\} , 
$$
}
$\thetav = \thetav + \eta_{\theta} J(\thetav),~~~~$ with step-size $\eta_{\theta}$;
}
\caption{ARSM policy gradient for reinforcement learning with a discrete-action space of $C$ actions.} \label{alg:ARM}}

\end{algorithm}

\end{document}